%% file: paper.tex
\documentclass[journal]{IEEEtran}

\usepackage[caption=false]{subfig}
\usepackage{cite}
\usepackage{graphicx}

\usepackage{algorithm}
\usepackage[noend]{algpseudocode}

\usepackage{algorithmicx}
\usepackage{etex}

\usepackage{amssymb}
\usepackage{mathrsfs}
\usepackage{amsmath}
\usepackage{multirow}

\usepackage{amsthm}

\usepackage[sans]{dsfont}
\usepackage{stmaryrd}
\usepackage{pifont}
\usepackage{wasysym}

\usepackage{array}
\usepackage{url}
\usepackage{colortbl}
\usepackage{caption}
\usepackage{color}
\usepackage{xcolor}
\usepackage{diagbox,pict2e}
\usepackage{capt-of}
\usepackage{float}

\algnewcommand{\LeftComment}[1]{\Statex \(\triangleright\) #1}

\makeatletter
\newcommand{\algmargin}{\the\ALG@thistlm}
\makeatother
\newlength{\whilewidth}
\algdef{SE}[parWHILE]{parWhile}{EndparWhile}[1]
{\parbox[t]{\dimexpr\linewidth-\algmargin}{%
		\hangindent\whilewidth\strut\algorithmicwhile\ #1\ \algorithmicdo\strut}}{\algorithmicend\ \algorithmicwhile}%
\algnewcommand{\parState}[1]{\State%
	\parbox[t]{\dimexpr\linewidth-\algmargin}{\strut #1\strut}}

\let\emptyset\varnothing

\def\cC{\mathcal C}

\def\cS{\mathcal S}

\def\cP{\mathcal P}

\def\bS{\bf S}

\newtheorem{innercustomthm}{Theorem}
\newenvironment{customthm}[1]
{\renewcommand\theinnercustomthm{#1}\innercustomthm}
{\endinnercustomthm}

\newtheorem{innercustomlem}{Lemma}
\newenvironment{customlem}[1]
{\renewcommand\theinnercustomlem{#1}\innercustomlem}
{\endinnercustomlem}

\newtheorem{lemma}{Lemma}
\newtheorem{theorem}{Theorem}

\newcommand\scalemath[2]{\scalebox{#1}{\mbox{\ensuremath{\displaystyle #2}}}}

\newcolumntype{x}[1]{>{\centering\arraybackslash\hspace{0pt}}m{#1}} 
\newcounter{magicrownumbers} 

\usepackage{pifont}
\newcommand{\cmark}{\color{green}\ding{51}}%
\newcommand{\xmark}{\color{red}\ding{55}}%
\usepackage{hhline}
\usepackage{pifont}
\renewcommand\appendixname{Supplementary Materials}
\begin{document}

\title{Autonomous Recharging and Flight Mission Planning for Battery-operated Autonomous Drones}

\author{Rashid Alyassi$^*$, Majid Khonji$^*$, Areg Karapetyan, Sid Chi-Kin Chau, {\em Senior Member}, {\em IEEE}, Khaled Elbassioni and Chien-Ming Tseng
\thanks{This work was supported by the Khalifa University of Science and Technology under Award Ref. CIRA-2020-286.}
\thanks{R. Alyassi, M. Khonji, K. Elbassioni and A. Karapetyan are with the EECS Department, Khalifa University, Abu Dhabi, UAE. (e-mails: \{rashid.alyassi, majid.khonji, khaled.elbassioni, areg.karapetyan\}@ku.ac.ae)}
\thanks{S. C.-K. Chau is with the Australian National University. (e-mail: sid.chau@anu.edu.au)}
\thanks{C.M. Tseng is with Ubiquiti Inc, Taiwan.}
\thanks{$^*$These authors contributed equally to the paper.}

 }

\maketitle


\input{abstract}
\input{intro}
\input{related}

\input{formulation}

\input{battery}

\input{model2}

\input{sim}

\input{app}

\input{charging}
\input{concl}

\bibliographystyle{IEEEtran}
\bibliography{paper}
\input{appendix}

\end{document}

%% file: abstract.tex
\begin{abstract}
Unmanned aerial vehicles (UAVs), commonly known as drones, are being increasingly deployed throughout the globe as a means to streamline monitoring, inspection, mapping, and logistic routines. When dispatched on autonomous missions, drones require an intelligent decision-making system for trajectory planning and tour optimization. Given the limited capacity of their onboard batteries, a key design challenge is to ensure the underlying algorithms can efficiently optimize the mission objectives along with recharging operations during long-haul flights. With this in view, the present work undertakes a comprehensive study on automated tour management systems for an energy-constrained drone: (1) We construct a machine learning model that estimates the energy expenditure of typical multi-rotor drones while  accounting for real-world aspects and extrinsic meteorological factors. (2) Leveraging this model, the joint program of flight mission planning and recharging optimization is formulated as a multi-criteria Asymmetric Traveling Salesman Problem (ATSP), wherein a drone seeks for the time-optimal energy-feasible tour that visits all the target sites and refuels whenever necessary. (3) We devise an efficient approximation algorithm with provable worst-case performance guarantees and implement it in a drone management system, which supports real-time flight path tracking and re-computation in dynamic environments. (4) The effectiveness and practicality of the proposed approach are validated through extensive numerical simulations as well as real-world experiments.
\end{abstract}\def\abstractname{Note to Practitioners}

%% file: intro.tex
\section{Introduction}
\IEEEPARstart{W}{ithin} future smart cities and ecosystems, autonomous drones are often envisaged as key-enabling technologies that would bolster and refine a range of vital services and procedures, including environmental surveillance, search and rescue operations, traffic monitoring and logistics~\cite{AFRBook2010,SebbaneBook2015}. This outlook rests on drones' salient properties: (1) {\em Energy-efficiency}: Small drones typically consume less energy per package-km than delivery trucks \cite{SSOLMC18drones}. They are particularly energy-efficient for transporting lightweight items in short trips (within 4 km), whereas ground vehicles are useful for carrying heavier objects over long distances.
(2) {\em Agility}: Unlike on the ground, there is little restriction and fewer obstacles in the sky, hence drones can travel across space in straight paths with nimble navigation.
(3) {\em Swiftness}: Aerial transportation is usually not hampered by traffic congestion. Thus, the flight duration  is mostly reflected by the distance traveled. Drones can also be rapidly launched via catapults and drop payloads by parachutes in response to time-critical situations.
(4) {\em Safeness}: As there is no on-board human operator, UAVs are particularly appealing for mission tasks that are hazardous, contagious or lethal. Furthermore, in applications such as transport of medical supplies or disaster management, aerial drones may conduce to saving lives. (5) {\em Low-cost}: Drone technologies have matured over time and price has dropped due to economies of scale. The adoption of customized drone systems can yield notable cost reduction in a number of applications, including field spraying, surveillance in precision agriculture, monitoring of difficult-to-access infrastructure as well as parcel deliveries~\cite{otto2018optimization}.

Nevertheless, drones are plagued with several operational challenges, such as limited battery endurance, meagre loadability and sensitivity to ambient factors~\cite{SSOLMC18drones, CHEN2021102214}. In particular, typical drones are only suitable for \textit{short-range trips}, which restrains their applicability in persistent, long-distance mission tasks. Incidentally, though drones are expected to travel within certain high altitudes, they are largely susceptible to \textit{wind conditions} (especially in light of their sheer weight).

Most of the existing planners for optimizing UAV routing focus on short-distance (reachable within the maximum flight time), recurrent tours that can be performed in parallel by a fleet of drones (e.g., for last-mile package deliveries). While beneficial on their own, these methods are not readily amenable to fuel-constrained, distant flight trips which demand a battery endurance model and entail an additional overhead for managing the recharging decisions. Such flights are often associated with \textit{solitary or non-parallelizable} missions involving monolithic workflows or interdependent subtasks. One practical example concerns police patrols, which are typically discerned based on the pre-assigned monitoring area or district. Currently, UAE is actively utilizing drones to accompany police cars or independently monitor specific sites of interest~\cite{uaedrone4400,while2021urban}. Another example relates to drone-based postal/governmental services being tested in UAE, France, UK, and Japan~\cite{frpostal, ukpostal,jppostal, uaepostal}. As is customary in these procedures, UAV's flight trajectory would likely consist of a sequence of visits to designated locations, wherein the shipment can undergo the prescribed successive pipeline for authentication and/or certification (e.g., stamp collection for a permit/document).

Without a reliable and inclusive battery consumption model in place, long-distance drone routing is prone to disruptions and exorbitant operational costs. Arguably, an ill-estimated battery charge level (e.g., owing to overlooked adversarial weather conditions) might degrade service quality or even lead to mission failure (e.g., when a drone runs out of charge before reaching a charging station or returning to the depot). On the other hand, frequent recharging detours might incur unnecessary delays and excessive energy expenditure.

In response to these demerits, we synthesize and experimentally validate an effective multifaceted tour management framework for energy-constrained, long-haul drone routing applications. Specifically, the present study complements and advances the relevant literature with the following four-fold contributions:
\begin{itemize}
	\item[\small\ding{228}] Drawing on intensive experimentation and analysis, we design a first-order regression model for estimating energy consumption of typical multi-copter UAVs. The model accommodates wind speed and direction, UAV motion kinematics and payload mass, allowing for sufficiently accurate yet computationally inexpensive battery performance estimation in complex real-world environments (within 5\% deviation error, as validated on three different drones).
	\item[\small\ding{228}] With the established model, the energy-constrained tour management problem is formulated as a multi-objective extension of ATSP, in which a drone is tasked with visiting the chosen (possibly distant) target locations while maintaining battery state-of-charge (SoC) within limits. The objective function aims at minimizing the total trip duration; the flight time plus the duration of recharging operations. 
	\item[\small\ding{228}] We devise an efficient \textit{approximation algorithm} inducing near-optimal (within an asymptotic constant factor) tandem of flight mission decisions and charging strategies. The average-case performance and scalability of the algorithm are demonstrated through numerical simulations.
	\item[\small\ding{228}] We implement the proposed approach in a drone management system that supports real-time flight path tracking and re-computation in dynamic environments. Subsequently, simulation studies and real-world experiments are provided to corroborate the effectiveness and practicality of the featured planner.
\end{itemize}

%% file: related.tex
\section{Related Work}


The extant literature on drones can be thematically organized into two major threads: (1) low-level transient control of flight operations, for instance, modulating propellers and maintaining balance through PID and MPC controllers~\cite{siegwart2004pid,bregu2016dronecontrol}, and (2) high-level planning and management of drone missions, for example, obstacle avoidance, localization and mapping, and trajectory planning~\cite{SNSbook}. In the latter theme, most prior research on UAV routing, such as the works in~\cite{7513397, JEONG2019220, SONG2018418, 9151388, Torabbeigi2020, CHENG2020364}, focused on short-distance fleet-based flight scenarios intended primarily for last-mile parcel deliveries. To overcome the range barrier, the studies in~\cite{sundar2013algorithms, 9099809, 8842627} propose different extended setups, involving mobile assisting platforms or multiple spatially distributed battery charging/swapping stations. Yet, the approaches developed therein lack optimality guarantees (except the one in~\cite{sundar2013algorithms}) and leave unexplored the potential of optimizing the duration of recharging operations (hence the consumed energy). Table~\ref{tab:paper-compare} provides a further comparison between the present paper and the aforementioned studies. 

    
    \begin{table*}[t]
    \renewcommand{\arraystretch}{1}
    	\centering
    	\resizebox{0.99\textwidth}{!}{
    		\begin{tabular}{|l|c|c|c|c|c|c|c|}
    			\hline
    			& \begin{tabular}[c]{@{}c@{}} Problem Setup\end{tabular} & \begin{tabular}[c]{@{}c@{}}Power \\Consumption\\ Estimation\end{tabular} & \begin{tabular}[c]{@{}c@{}}Recharging\\ Optimization\end{tabular} & \begin{tabular}[c]{@{}c@{}}Experimental\\ Validation\end{tabular} & \begin{tabular}[c]{@{}c@{}}Optimality\\ Guarantees\end{tabular} & \begin{tabular}[c]{@{}c@{}}Intended Flight Trajectory \end{tabular} & 
    			\begin{tabular}[c]{@{}c@{}}Application \end{tabular} \\\hline
    			Sundar and Rathinam~(2013)~\cite{sundar2013algorithms}& Single UAV, multi-CS &  \xmark& \xmark & \xmark & \cmark  &  Long-distance, single trip & Multi-domain\\ \hline
    			Dorling et. al.~(2017)~\cite{7513397}&  Multi-UAV, single CS & White-box model &  \xmark&  \xmark & \xmark & Short-range, multi-trip & Logistics\\ \hline
    			Song et. al.~(2018)~\cite{SONG2018418}& Multi-UAV, multi-CS & \xmark & \xmark & \xmark & \xmark & Short-range, multi-trip & Logistics\\ \hline
    			{Vivaldini et.al.~(2019)~\cite{Vivaldini2019}}& Single UAV, single CS & \xmark & \xmark & \cmark & \xmark  & Short-range, single trip & Monitoring\\ \hline
    			{Jeong et.al.~(2019)~\cite{JEONG2019220}}& Single UAV, mobile BSS & White-box model & \xmark & \xmark & \xmark  & Short-range, multi-trip & Logistics\\ \hline
    			Shao et. al.~(2020)~\cite{9099809}&  Single UAV, multi-BSS & \xmark &  \xmark& \xmark & \xmark & Long-distance, single trip & Logistics\\ \hline
    			Ribeiro et. al.~(2020)~\cite{8842627}& Multi-UAV, multi-CS & \xmark & \xmark & \xmark & \xmark & Long-distance, single trip & Monitoring \& inspection\\ \hline
    			Torabbeigi et. al.~(2020)~\cite{Torabbeigi2020}& Multi-UAV, single CS & Black-box model & \xmark & \xmark & \xmark &  Short-range, single trip & Logistics\\ \hline
    			Cheng et. al.~(2020)~\cite{CHENG2020364}&  Multi-UAV, single CS& White-box model & \xmark & \xmark & \xmark & Short-range, multi-trip & Logistics \\  \hline
    			{Huang et.al.~(2021)~\cite{9151388}}& Single UAV, mobile TLP& \xmark & \xmark &  \xmark &  \xmark  & Short-range, multi-trip & Logistics\\ \hline\hline
    			\multicolumn{1}{|c|}{\multirow{2}{*}{Present work}}&  Single UAV, multi-CS & \multirow{2}{*}{Black-box model} & \multirow{2}{*}{\cmark} & \multirow{2}{*}{\cmark} & \multirow{2}{*}{\cmark} & \multirow{2}{*}{Long-distance, single trip} & \multirow{2}{*}{Multi-domain}\\  
    			\multicolumn{1}{|c|}{}&  {Multi-UAV, shared multi-CS (see Sec.~\ref{exte})} &  &  &  &  &  & \\\hline
    		\end{tabular}
    	}
    	\caption{A comparative summary of related literature on UAV routing and mission planning problems. The acronyms CS, BSS and TLP stand for charging station, battery swap station and takeoff and landing platform, respectively.}
    	\label{tab:paper-compare}
    \end{table*}
    
Meanwhile over the recent past, several general-purpose planners for autonomous robots, such as Kongming~\cite{li2008generative}, p-Sulu \cite{ono2013probabilistic}, COLIN~\cite{coles2008planning} or ScottyActivity~\cite{fernandez2018scottyactivity}, have emerged in the Artificial Intelligence  community. An important characteristic of COLIN and ScottyActivity is that they do not require time discretization. This is essential for efficient planning in scenarios with long operational horizons and activities with multiple time scales. Although these approaches have significantly increased expressivity of the problems that could be modeled, there is a lack of theoretical understanding of solution quality and running time. In particular, they often rely on Heuristic Forward Search (e.g. as in
COLIN and ScottyActivity), and Integer Programming (e.g., as in p-Sulu), both of which are provably unscalable in many domains. This work explores a different approach by exploiting problem structure with a deeper understanding from a theoretical perspective. 

The problem under study is linked to several variants of the Vehicle Routing Problem (VRP), namely Solar-VRP~\cite{LGMJKLY2016ev}, Hybrid-VRP~\cite{ckt2016dm}, Green-VRP~\cite{erdougan2012green}, and Electric-VRP~\cite{ham2021electric}. Solar-VRP seeks to optimize the route and speed of a solar-powered electric vehicle in a single-source single-destination tour such that power consumption is balanced by harnessed solar energy. In Hybrid-VRP, the objective instead is to minimize the fuel consumption of a hybrid (electric and fuel powered) vehicle by modulating its driving mode based on trip information and finding an optimal path considering intermediate filling and charging stations. Green-VRP, Electric-VRP and variants thereof are concerned with optimal routing and refueling planning of a fleet of alternative fuel-powered vehicles considering the associated environmental and financial tolls. Unlike the setting studied here, the vehicle energy/fuel consumption rate is assumed constant in the preceding line of studies. For current purposes, we cast the energy-constrained drone tour management program as an extended version of the Fuel-constrained UAV Routing Problem (FCURP) studied in~\cite{sundar2013algorithms}. Therein, the authors develop an approximation algorithm for FCURP, based on the approach in~\cite{KMM2011gsp} for the special case of the problem with symmetric travel costs. Building upon these methods, we devise an asymptotic constant-factor approximation algorithm for multi-objective FCURP and validate its practicality experimentally through real-world trials.

    


Although power consumption modeling and estimation of electric vehicles has been extensively studied, the research focus revolved primarily around ground vehicles~\cite{ganti2010greengps,stefan2013modularecev,alesaandn2002statmdlfc,cmtseng2015pardte,tc2016dte, eugene2013rtbattery}. Deviating from the latter, UAVs exhibit certain unique characteristics that entail new challenges (e.g., the impact of wind is more substantial for drone flight). As surveyed in~\cite{zhang2021energy}, the methods developed for drones can be broadly categorized into two types: model-based white-box and black-box. The former approach hinges on explicit theoretical (in a sense ``microscopic'') behavior model of a drone that comprehensively characterizes the motor performance, aerodynamic environment, and battery systems. However, producing a reliable white-box model often requires a large amount of data for calibration as well as specific particulars of the drone. For example, the aerodynamic parameters such as propeller and motor efficiencies, drag coefficients could be cumbersome to compute accurately without resorting to sophisticated experimental setups like wind tunnel. In contrast, a model-agnostic black-box method that relies on generic statistical techniques can estimate battery endurance of a drone with only a small set of measurable variables and parameters. In the sections to follow, we design a simple yet accurate power consumption regression model for multi-copter UAVs which is then leveraged for joint flight mission planning and recharging optimization.

%% file: formulation.tex
\section{Problem Statement}\label{secprob}

This section formalizes the problem of joint flight mission planning and recharging optimization for battery-operated autonomous drones. We propose a multi-criteria objective function that seeks to complete a flight tour mission for a set of target sites in the shortest time while minimizing duration of intermediate recharges. 

As illustrated in Fig.~\ref{fig:DFP}, consider a set of sites of interest, denoted by $\mathcal{S}$, that a drone needs to visit (e.g., sites for measurements or patrolling or drop-off locations of parcels), and a set $\mathcal{C}$ of charging station locations where a drone can refuel its battery. Let $v_0$ correspond to the base location of a drone and $V\triangleq\mathcal{S}\cup\mathcal{C}\cup\{v_0\}$.

\begin{figure}[t]
	\centering
	\includegraphics[width=0.73\columnwidth]{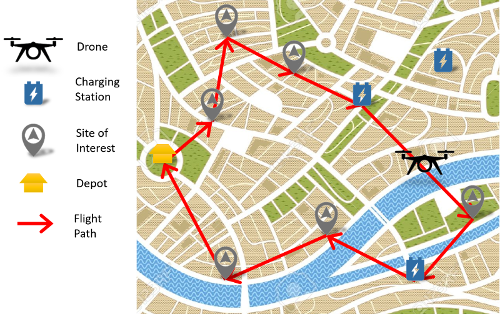} 
	\caption{A drone flight mission plan with charging stations.} \label{fig:DFP}
\end{figure}

Given a pair of locations $u, v \in V$, we denote the designated direct flight path from $u$ to $v$ by $\ell(u,v)$, and the corresponding flight time by $\tau(u,v)$. As is customary, we assume that, while in route, the drone first ascends vertically to a desired altitude, then travels in a straight path, and finally descends to the destination vertically. Let ${\sf E}\big(\ell(u,v), \tau(u,v)\big)$ be the energy consumption required for the drone to fly along $\ell(u,v)$ in a time span of $\tau(u,v)$. Here, ${\sf E}(\cdot,\cdot)$ is an increasing function that maps the combination of flight path $\ell(u,v)$ and flight time $\tau(u,v)$ to the required amount of energy. As exemplified in Section~\ref{sec:battery}, ${\sf E}(\cdot,\cdot)$ can be estimated through a first-order regression model.

We represent the \textit{charging strategy} by a function $b(\cdot):\mathcal{C}\mapsto{\Bbb R}$ that maps a charging station to an amount of energy to be recharged.  When recharging the battery at $u\in\mathcal{C}$, let the incurred charging time be $\tau_{\rm c}(b(u))$. Define by $\eta_{\rm c}\le1$ and $\eta_{\rm d}\ge1$ the charging and discharging efficiency coefficients, respectively. Once at a charging station $u$, the drone recharges its battery by an amount of energy denoted by $\eta_{\rm c} b(u)$. Whereas when flying between two sites $u,v \in V$, it drains $\eta_{\rm d}{\sf E}\big(\ell(u,v), \tau(u,v)\big)$ amount of energy from the battery.  

A {\em flight mission plan} is denoted by $\mathcal{F}$, which is a tour starting and terminating at $v_0$ that consists of a sequence of locations in $\mathcal{S} \cup \mathcal{C} \cup \{ v_0 \}$. Set $\mathcal{F}_k$ to be the $k$-th location in $\mathcal{F}$, then $\mathcal{F}_1 = \mathcal{F}_{|\mathcal{F}|} = v_0$, and let $x_k$ be the state-of-charge (SoC) when reaching $\mathcal{F}_k$.
We require the SoC to stay within a feasible range $[\underline{B},\overline{B}]$. The lower bound of SoC, $\underline{B}$, ensures sufficient residual energy for the drone to return to the base in case of emergency. The initial SoC is conventionally set to $x_0 = \overline{B}$.

With the above notations, the drone flight mission planning with recharging problem ({\sf DFP}) is formulated as
{\small\begin{align}
({\sf DFP})\quad &  \min_{\mathcal{F}, b(\cdot),x} \ \sum_{k=1}^{|\mathcal{F}|-1} \tau(\mathcal{F}_k,\mathcal{F}_{k+1}) + \sum_{k=1:\mathcal{F}_k \in \mathcal{C}}^{|\mathcal{F}|} \tau_{\rm c}(b(\mathcal{F}_{k})) \nonumber\\
\text{s.t.} \quad&  \mathcal{F}_1 = \mathcal{F}_{|\mathcal{F}|} = v_0\\
& \mathcal{S} \subseteq \mathcal{F} \subseteq \mathcal{S} \cup \mathcal{C} \cup \{ v_0 \}\\
& x_k = \left\{
\begin{array}{@{}l@{}}
x_{k-1} - \Psi_{k,k+1},  \mbox{\ if\ } \mathcal{F}_k \in \mathcal{S} \\
x_{k-1} + \eta_{\rm c} b(\mathcal{F}_{k+1}) -\Psi_{k,k+1},  \mbox{\ if\ } \mathcal{F}_k \in \mathcal{C}\\
\end{array}
\right. \label{eq:x}\\
& \underline{B} \le x_k \le \overline{B}, \ x_0 = \overline{B}\,,\
\end{align}}where $\Psi_{k,k+1} \triangleq \eta_{\rm d}{\sf E}\Big(\ell(\mathcal{F}_k,\mathcal{F}_{k+1}), \tau(\mathcal{F}_k,\mathcal{F}_{k+1})\Big)$.

DFP aims to find a flight mission plan $\mathcal{F}$ together with a charging strategy $b(\cdot)$ that minimizes the total trip time, consisting of the flight time plus the recharging time, while maintaining the SoC within permissible limits. 
The difficulty of {\sf DFP} is to balance the flight decisions and charging decisions. On one hand, a flight mission plan needs to consider the requirement of completing the tour goals in minimal total trip time. On the other hand, it needs to be able to reach a charging station in case of insufficient SoC.

For further practical adaptations, the formulation of {\sf DFP} can be extended to incorporate a variety of pragmatic mission planning factors, such as restrictions of no-fly zones and altitude as well as wind speed forecast information. Users can additionally weigh in with application-specific preferences, such as mission completion deadline or maximum payload weight. In Section~\ref{exte}, we highlight several immediate extensions for future work and sketch their solution methodologies.


%% file: battery.tex
\section{Power Consumption Model}\label{sec:battery}

In order to accurately optimize the flight missions of drones, this section develops a practical battery endurance estimator for UAVs. First, we perform extensive experiments considering various flight scenarios on multiple drones to generate training data and examine energy consumption covariates. We then develop a nine-term regression model and train it for each drone. Lastly, we appraise the model fidelity through field experiments.



\subsection{Experimental Setup and Scope}

We evaluate the power consumption of three commercial drone models, namely 3DR Solo\footnote{\url{https://3drobotics.com/support}}, DJI Matrice 100, and DJI Matrice 600 Pro\footnote{\url{https://www.dji.com/products/enterprise?site=brandsite&from=nav#drones}}, which appear in Fig.~\ref{fig:drone-solo} (see Table~\ref{tab:drone-solo} in the supplementary materials for their specifications). The drones support developer kits, which allowed us to extract the sensor readings and program the flight paths. The onboard barometer and GPS sensors served for measuring the 3-dimensional movements of a drone. The (ground) speed and position of a drone were obtained through GPS and IMU modules, while the altitude was obtained through barometer and GPS. 

\begin{figure}[!h]
	\includegraphics[width=.3\linewidth, height=20mm]{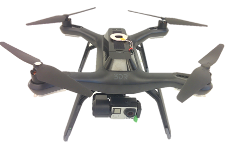}  
	\includegraphics[width=.32\linewidth, height=20mm]{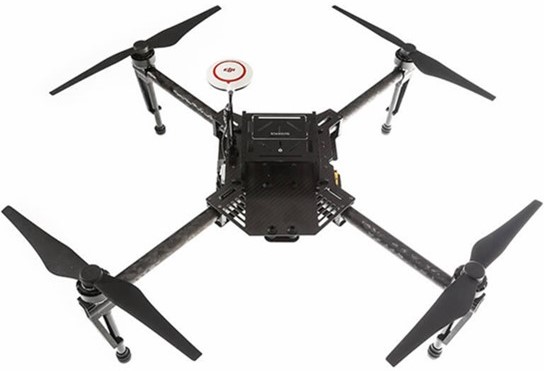}  
	\includegraphics[width=.32\linewidth, height=20mm]{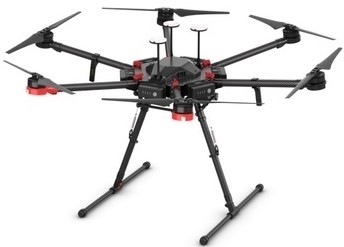}   
	\caption{Left: 3DR Solo. Middle: DJI Matrice 100. Right: DJI Matrice 600 Pro.}  \label{fig:drone-solo}
\end{figure}


To construct a detailed and practical battery endurance model, we examine the following four compound factors:

\begin{enumerate}
    
    \item {\bf Motion -} We further categorize drone motion into three types: hovering, horizontal and vertical. 
    
    \item {\bf Weight -} The total weight of a drone including the carried payload.
    
    \item {\bf Wind -} Wind may benefit the power consumption in some cases and incur resistance in other cases. To distinguish, both ground speed and direction are considered.

    \item {\bf Altitude -} Flight altitude correlates with the fluid density of air, which in turn may affect rotor thrust.

\end{enumerate}



\subsection{Experimentation Results and Analysis}\label{subsec-power_exp}
\noindent\textbf{\textit{Analyzing the Impact of Motion:}} To this end, four experiments were conducted, with the first three on 3DR Solo and the last on DJI Matrice 600, which results are detailed below. Figs.~\ref{fig:sensordata} and~\ref{fig:m600_seonsor} depict the recorded data traces.

\noindent{\scshape{Experiment 1}:}
 To assess the baseline power consumption, the test drone was hovered in the air without any movement. Note that drones may slightly drift around the takeoff location due to deviation error of GPS modules, hence the speed data smaller than 0.5 m/s is filtered out. From the recorded data, we observe that the drone can maintain a sufficiently steady flying altitude with steady power consumption.

\noindent{\scshape Experiment 2:}
 The test drone ascended and descended repeatedly, producing time-series data that allowed computing its vertical acceleration and speed. We observe larger power fluctuations due to vertical movements. Power consumption increases slightly when the drone ascends steadily.

\noindent {\scshape Experiment 3:}
 The test drone moved horizontally without altering its altitude in this experiment. The GPS data comprises of speed and course angle of the drone. We also gathered average wind speed and direction using a wind speed meter. We observe smaller power fluctuations due to horizontal movements as compared to vertical ones.
 
 \noindent {\scshape Experiment 4:} To test the effect of speed on a larger drone, we programmed DJI Matrice 600 to move horizontally at a fixed altitude under three different speeds. As revealed by Fig.~\ref{fig:m600_seonsor}, no notable difference in power consumption was observed. 

\begin{figure}[htb!]
    \includegraphics[width=\linewidth]{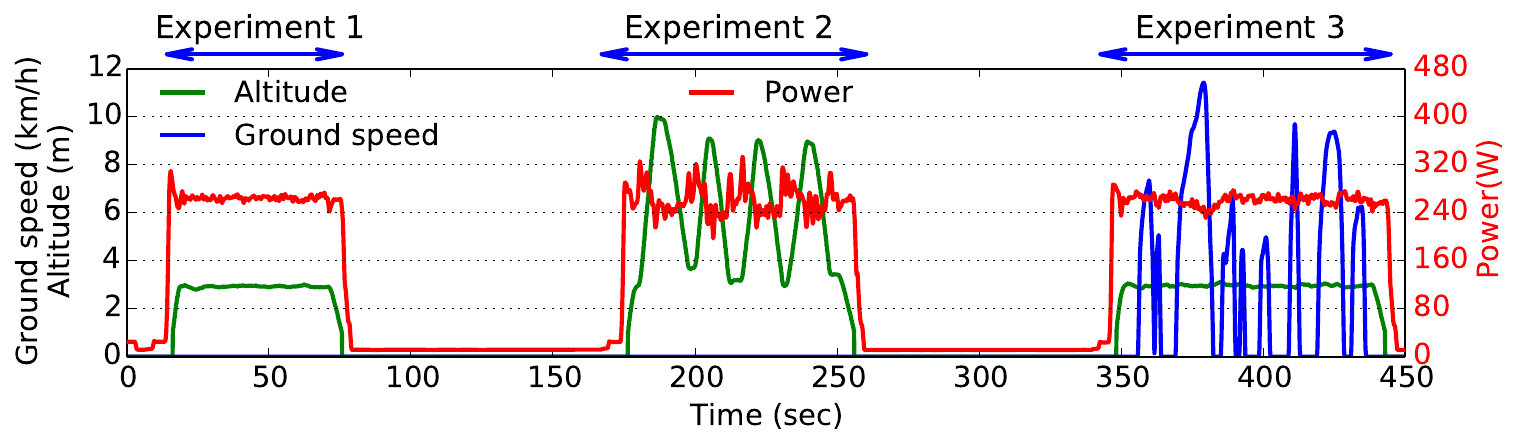}
	\caption{Power consumption of 3DR Solo for different motions.} 
	\label{fig:sensordata}
\end{figure}


\noindent\textbf{\itshape Analyzing the Weight-induced Impact:} Here, two experiments were performed on DJI Matrice 600 and 3DR Solo with varying payload weights as illustrated in Figs.~\ref{fig:weightexp} and~\ref{fig:m600_seonsor}.

\noindent {\scshape Experiment 5:} To obtain the baseline power consumption, we set 3DR Solo to hover in the air without any movement while carrying three different weights. As inferred from Fig.~\ref{fig:weightexp}, the observed power consumption increases almost linearly with the payload weight. Note that the maximum loading capacity is 500g for 3DR Solo.

\noindent {\scshape Experiment 6:} To examine the joint impact of motion and weights, DJI Matrice 600 was dispatched to fly horizontally under three different payload and speed settings. In line with the preceding analysis, nearly linear power consumption growth was recorded with respect to payload weight during stationary motion, whereas for horizontal acceleration we observe rather comparable spikes in power consumption.

\begin{figure}[!ht]
	\includegraphics[width=\linewidth]{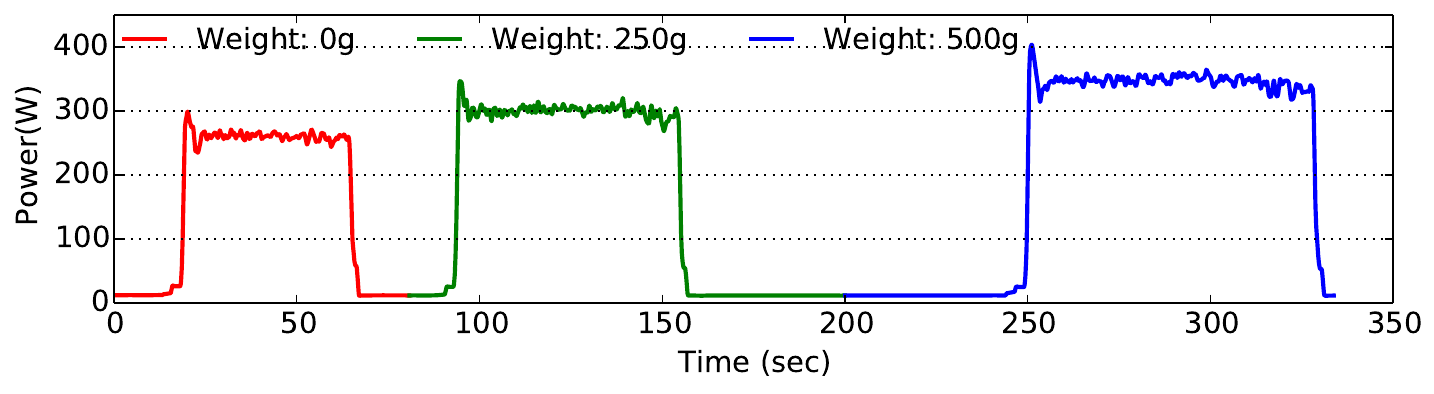}
	\caption{Power consumption of 3DR Solo for different payloads.}  
	\label{fig:weightexp}
\end{figure}
\begin{figure}[ht]
    \centering
    \includegraphics[width=0.97\linewidth]{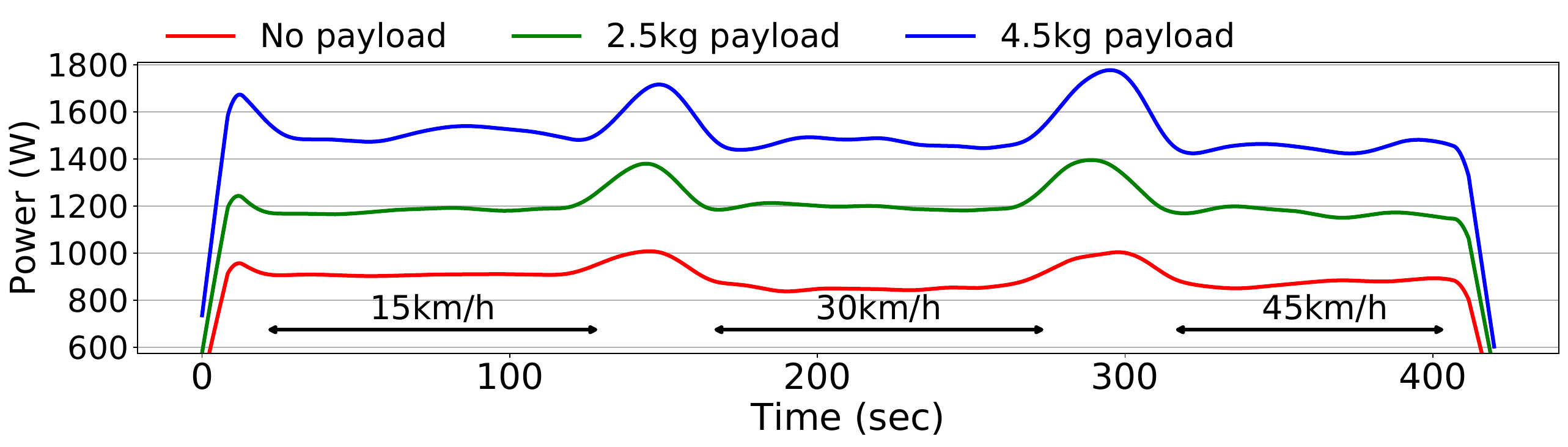}
    \caption{Power consumption of DJI Matrice 600 under different payloads and horizontal speeds.}
    \label{fig:m600_seonsor}
\end{figure}

\noindent\textbf{\itshape Exploring the Effect of Wind:} For exhaustive analysis, four test runs were performed under different wind speeds and directions.  
The experiment was conducted at the same location but on days with different wind conditions. The wind direction and average speed were measured with a wind meter.

\noindent{\scshape Experiment 7:} Fig.~\ref{fig:windtexp} shows the battery power consumption of 3DR Solo under different wind speeds when flying into a {\em headwind} and {\em tailwind} at maximum ground speed (18 km/h). We observe decreased power consumption when flying into a {\em headwind}, which is due to the increasing thrust by {\em translational lift} when the drone transitions from hovering to forward flight. When flying into a {\em headwind}, translational lift increases as the relative airflow over the propellers increases, resulting in less power consumption to hover the drone \cite{FAABook2012}. However, when the wind speed exceeds a certain limit, the aerodynamic drag may outweigh the benefit of translational lift. The drone speed is relatively slow in our setting, even at its peak. Hence, flying into a {\em headwind} is likely more energy-efficient for 3DR Solo.

\begin{figure}[!h]
	\includegraphics[width=\linewidth]{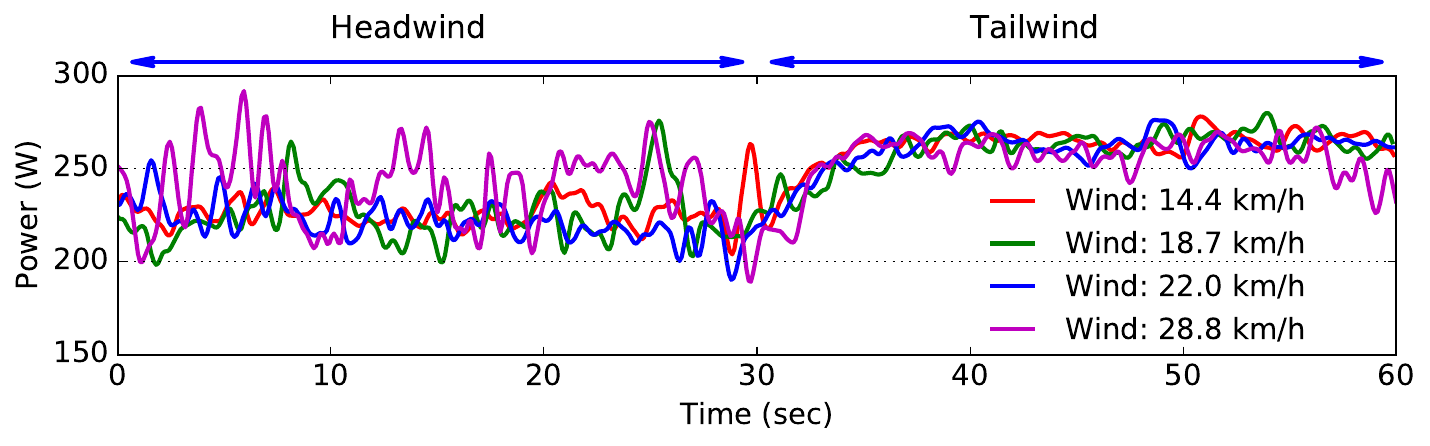}
	\caption{Power consumption of 3DR Solo under different wind conditions.}  
	\label{fig:windtexp}
\end{figure}

\noindent\textbf{\itshape Assessing the Influence of Altitude:} To quantify the impact of air density on battery performance, we performed multiple tests on DJI Matrice 600 as elaborated below.  

\noindent {\scshape Experiment 8:} The test drone was hovered for around 2 minutes with zero ground speed at a height of 50 and 110 meters (FAA maximum permitted altitude is 120 meters ). As seen from Fig.~\ref{fig:altitudetexp}, the observed variation in power consumption is merely $1\%$. While at higher altitudes the effect might be more substantial, given the adopted FAA regulations, the incurred impact is assumed constant. 

\begin{figure}[!h]
    \center
    {\includegraphics[width=\linewidth]{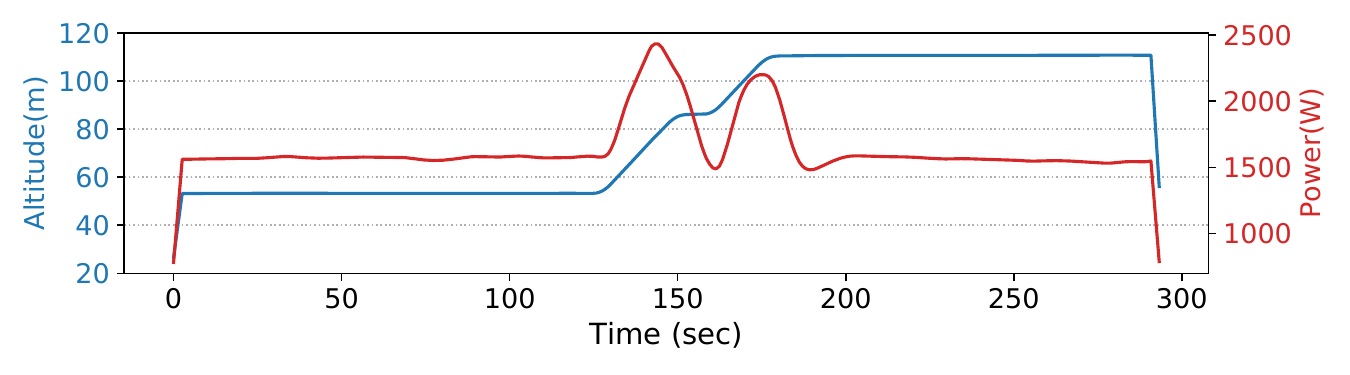}}  
	\caption{Power consumption of DJI Matrice 600 while hovering at low and high altitudes.}  
	\label{fig:altitudetexp}
\end{figure}

\subsection{Regression Model of Power Consumption for Drones}



This section introduces a general multivariate black-box model of power consumption for multi-copter drones. In essence, the battery consumption of an electric vehicle is determined by the total power needed to overcome the physical forces opposing its motion, which include the force for acceleration and aerodynamic drag force\cite{stefan2013modularecev,alesaandn2002statmdlfc, eugene2013rtbattery}.

Let $\hat{P}$ denote the estimated battery power consumption of a drone. Hinging on the insights gained in Section~\ref{subsec-power_exp} as well as prior models for electric vehicles, we express $\hat{P}$ as a linear combination of three composite forces parameterized by mass, acceleration, velocity and wind attributes. Mathematically, $\hat{P}$ takes the following form:
\begin{equation}\scalemath{0.75}{
\hat{P} =  
\begin{bmatrix}
\beta_1\\
\beta_2\\
\beta_3
\end{bmatrix}^{T} 
\begin{bmatrix}
\|\vec{v}_{xy}\|\\
\|\vec{a}_{xy}\|\\
\|\vec{v}_{xy}\|\|\vec{a}_{xy}\|
\end{bmatrix} 
 \mbox{+}
\begin{bmatrix}
\beta_4\\
\beta_5\\
\beta_6
\end{bmatrix}^{T}
\begin{bmatrix}
\|\vec{v}_{z}\|\\
\|\vec{a}_{z}\|\\
\|\vec{v}_{z}\|\|\vec{a}_{z}\|
\end{bmatrix}
 \mbox{+}
\begin{bmatrix}
\beta_7\\
\beta_8\\
\beta_9
\end{bmatrix}^{T}
\begin{bmatrix}
 m \\
\vec{v}_{xy} \cdot \vec{w}_{xy} \\
1
\end{bmatrix}
\label{eqn:totalmodel}}
\end{equation}
\noindent where $\|\cdot\|$ denotes the magnitude of a vector, ${\beta}_{1},..., {\beta}_{9}$ are the regression coefficients to be calculated and
\begin{itemize}
\item $\vec{v}_{xy}$ and $\vec{a}_{xy}$ are the speed and acceleration vectors describing the horizontal movement of the drone,
\item $\vec{v}_{z}$ and $\vec{a}_{z}$ are the speed and acceleration vectors characterizing the vertical movement of the drone,
\item $m$ is the payload weight,
\item $\vec{w}_{xy}$ is the vector of wind dynamics in the horizontal surface.
\end{itemize}

Assuming uniform conditions in \eqref{eqn:totalmodel}, the total power consumption of a drone in time interval $\Delta t$ would then amount to $\hat{P} \cdot \Delta t$. We note that while the above model does not capture all detailed factors, it can provide relatively accurate estimations with low computational complexity, as verified in Secs.~\ref{subsec-reg_model} and~\ref{exp}.

\subsection{Model Evaluation}\label{subsec-reg_model}

Based on the data collected from the experiments in Section~\ref{subsec-power_exp}, power consumption predictive models were trained for each of the three test drones. The resulting regression coefficients for 3DR Solo, DJI Matrice 100 and Matrice 600 are as follows: $\beta_{\rm solo} =$  [-1.526, 3.934, 0.968, 18.125, 96.613, -1.085, 0.220, 1.332, 433.9], $\beta_{\rm M100} =$  [-2.595, 0.116, 0.824, 18.321, 31.745, 13.282, 0.197, 1.43, 251.7], and $\beta_{\rm M600} =$ [-1.777, 4.408, -0.038, 93.94, 1.362, -0.111, 140.46, 2.249, 0.0].

We assess the accuracy of produced estimations through two set of experiments. In the first, 3DR Solo and DJI Matrice 100 were programmed to perform vertical movements, then fly into a headwind and tailwind carrying different payload amounts while maintaining their altitude during the horizontal flight. As evidenced by Fig.~\ref{fig:deva1-2}, the predicted and the ground truth power consumption records match closely within 0.4\% error. The second experiment, conducted with DJI Matrice 600, employed a more complex flight scenario with multiple target locations and a recharging station as laid out in Section~\ref{exp}. According to the results, the estimated energy consumption was within 5\% deviation from the actual measurements.  

 \begin{figure}[h]
\center
\captionsetup[subfigure]{labelformat=empty}
\subfloat[]{  \includegraphics[trim=0 5.2mm 0 0, clip,width=.97\linewidth]{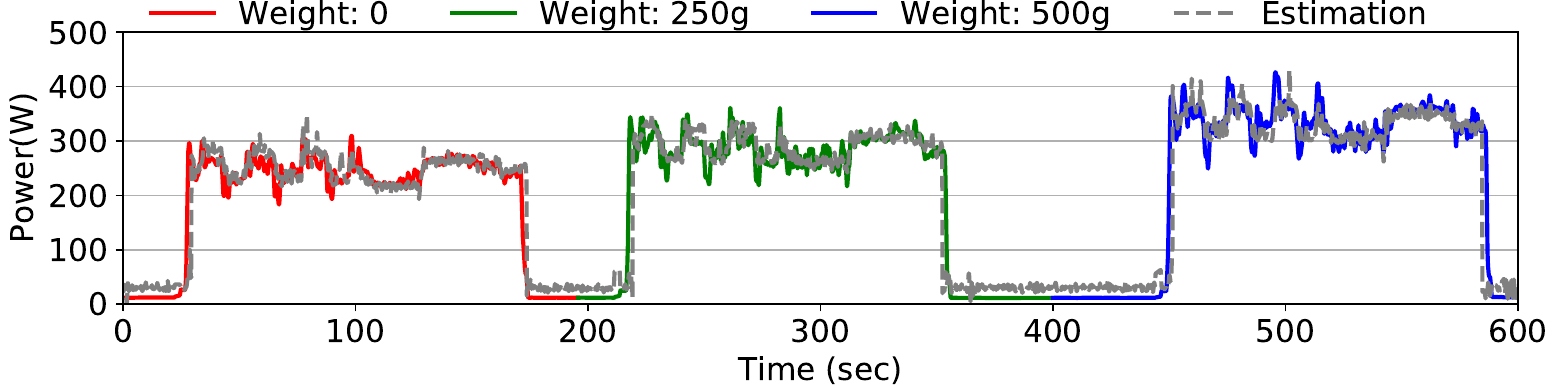}}\vspace*{-5mm}

\subfloat[]{  \includegraphics[width=.97\linewidth]{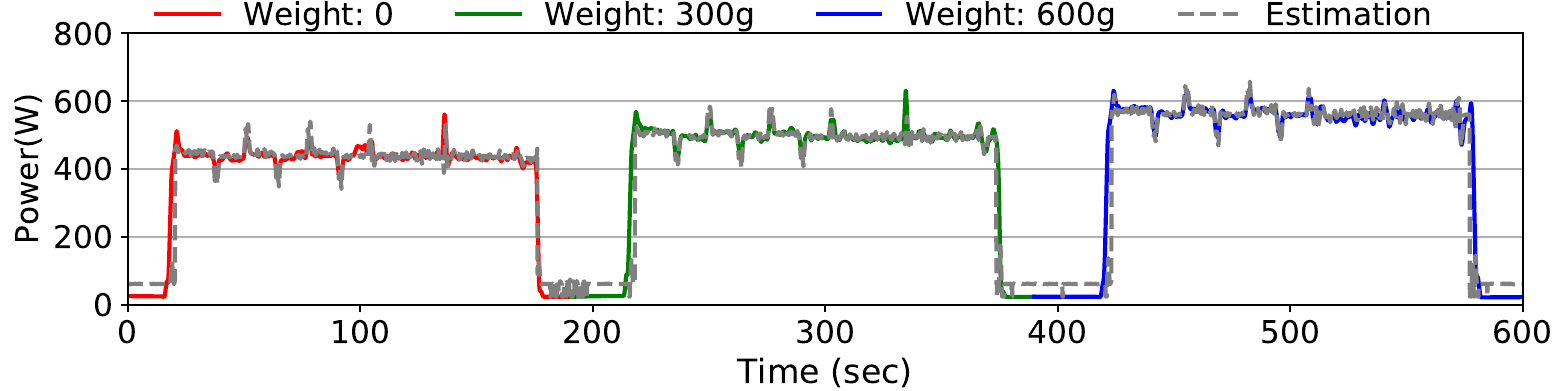}}
\caption{Measured and estimated power consumption of 3DR Solo (top) and DJI Matrice 100 (bottom). }
\label{fig:deva1-2}
\end{figure}

%% file: model2.tex
\section{Solution Methodology}

\subsection{Case with Uniform Speed and Steady Wind}\label{sec:ndl}

To provide efficient algorithms for {\sf DFP}, we first consider a basic setting under some realistic assumptions. Suppose that the horizontal speed of the drone is a uniform constant under steady wind condition, which will be relaxed in Section~\ref{sec:ext}. Then, the flight time $\tau(u,v)$ between two sites $u,v\in V$ is proportional to the length of flight path $\ell(u,v)$, denoted by $d(u,v)$.
The energy consumption model established in Section~\ref{sec:battery} implies that the function ${\sf E}\big(\ell(u,v), \tau(u,v)\big)$ is linear in the distance $d(u,v)$, and the charging time $\tau_{\rm c}(b(u))$ is linear in the amount of recharged energy $b(u)$. Thus, we assume the following linear objective functions:
\begin{align}
& \tau(u,v)=c_{a} d(u,v), \tau_{\rm c}(b(u))=c_{b}b(u),  \label{eqn:linearcost} \\
& {\sf E}\big(\ell(u,v),\tau(u,v)\big)=c_{f}(u,v) \cdot d(u,v), 
\end{align}
for some constants $c_{a},c_{b},c_{f}(u,v)>0$. Note that we allow $c_{f}(u,v)$ to be edge-dependent. This can model non-uniform environment\footnote{{Note that we consider the asymmetry in the energy consumption, due to wind direction, by assuming two different constants $c_{f}(u,v)$ and $c_{f}(v,u)$. 
}} for each $\ell(u, v)$, for instance, a path experiencing stronger wind is expected to have a larger constant $c_{f}(u,v)$.  
Denote the minimum and maximum values of $c_f$ by  $\underline{c}_{f}\triangleq \min_{(u,v)}c_{f}(u,v)$ and $\overline{c}_{f}\triangleq \max_{(u,v)}c_{f}(u,v)$, respectively.

In this paper, we concentrate mostly on long-distance trips (e.g., $\geq$ 4 km), for which the vertical landing and take-off operations usually constitute a small part of the whole flight (e.g., $<10$m vertically), hence account for only a marginal percentage of the total energy expenditure (e.g., $<1\%$). For exposition clarity, the energy consumption of these operations is assumed to be implicitly captured by $c_{f}(u,v) \cdot d(u,v)$, though the results can be easily extended to consider them explicitly.

For convenience of notation, for a flight mission plan $(\mathcal{F},b(\cdot))$, we write
$
\tau(\mathcal{F})\triangleq\sum_{k=1}^{|\mathcal{F}|-1} \tau(\mathcal{F}_k,\mathcal{F}_{k+1})$, $\tau_{\rm c}$$(b(\mathcal{F}))\triangleq\sum_{k=1:\mathcal{F}_k \in \mathcal{C}}^{|\mathcal{F}|} \tau_{\rm c}(b(\mathcal{F}_{k}))$ and define $d(\mathcal{F})\triangleq\sum_{k=1}^{|\mathcal{F}|-1} d(\mathcal{F}_k,\mathcal{F}_{k+1})$.

Under the aforementioned assumptions, the total charging time $\tau_{\rm c}(b(\mathcal{F}))$, in an optimal flight mission plan $\mathcal{F}$ with the corresponding charging strategy $b(\mathcal{F})$, is related to the total flight time $\tau(\mathcal{F})$ by the following lemma.

\begin{lemma}\label{lem:tot}
	In an optimal flight mission plan  $(\mathcal{F},b(\cdot))$, we have
	$$
	\underline{c}\cdot d(\mathcal{F})+c' \le 
	\tau(\mathcal{F}) + \tau_{\rm c}(b(\mathcal{F}))
	\le \overline{c}\cdot d(\mathcal{F})+c'
	$$
	where either 
	\begin{enumerate}
		
		\item[1)]
		$\underline{c}=\overline{c}=c_{a}$ and ${c}'=0$, or 
		
		\item[2)] $\underline{c}=c_{a}+\underline{c}_{f}c_{b}\frac{\eta_{\rm d}}{\eta_{\rm c}}$, $\overline{c}=c_{a}+\overline{c}_{f}c_{b}\frac{\eta_{\rm d}}{\eta_{\rm c}}$, and $c'=\frac{c_{b}}{\eta_{\rm c}}(\underline{B}-x_0)$. 
		
	\end{enumerate}
\end{lemma}
\begin{proof}
	See the supplementary materials.  
\end{proof}

In other terms, Lemma~\ref{lem:tot} characterizes the total charging time with respect to  total fight time in the optimal flight plan by the respective upper and lower bounds in terms of the distances to be traveled.

\begin{lemma}\label{lem:red}
	Given any feasible flight mission plan $(\mathcal{F},b(\cdot))$, there is another feasible flight mission plan $(\mathcal{F},b'(\cdot))$ such that 
	$$
	\tau_{\rm c}(b'(\mathcal{F})) \le \frac{\underline{B}-x_0}{\eta_{\rm c}}+\frac{\overline{c}_{f}\eta_{\rm d}}{\eta_{\rm c}} \cdot d(\mathcal{F})\,.$$
	Such a plan $(\mathcal{F},b'(\cdot))$ can be found in $O(|V|)$ time.
\end{lemma}	
\begin{proof}
	See the supplementary materials. (Algorithm~\ref{alg:fix-charge} below demonstrates how to construct  such a plan explicitly.)  
\end{proof}	

Both Lemma~\ref{lem:tot} and Lemma~\ref{lem:red} allow us to focus on minimizing the distance $d(\mathcal{F})$ instead of total trip time. 
Hence, we simplify {\sf DFP} such that the resulting problem's optimal solution is later shown to be within a constant factor from an optimal solution of {\sf DFP}. The simplified formulation ({\sf SDFP}) is defined as
\begin{align}
({\sf SDFP}) \quad  & \min_{\mathcal{F},x} \sum_{k=1}^{|\mathcal{F}|-1} \widehat{d}\big(\mathcal{F}_k,\mathcal{F}_{k+1}\big) \\
\mbox{s.t.\ }\quad & \mathcal{F}_1 = \mathcal{F}_{|\mathcal{F}|} = v_0 \\
& \mathcal{S} \subseteq \mathcal{F} \subseteq \mathcal{S} \cup \mathcal{C} \cup \{ v_0 \}\\
& x_k = \left\{
\begin{array}{@{}l@{}}
x_{k-1} -\eta_{\rm d}\widehat{d}(\mathcal{F}_k,\mathcal{F}_{k+1}),  \mbox{\ if\ } \mathcal{F}_k \in \mathcal{S}\\
\overline{B}, \quad \mbox{\ if\ } \mathcal{F}_k \in \mathcal{C}\\
\end{array}
\right. \label{eq2:x}\\
& \underline{B} \le x_k \le \overline{B}, \ x_0 = \overline{B},
\end{align}
\noindent and its explicit Mixed-Integer Linear Programming (MILP) formulation is provided in the supplementary materials. In {\sf SDFP}, we consider a modified distance function $\widehat{d}(\cdot, \cdot)$, which is defined as follows. Recall that  $V\triangleq\mathcal{S}\cup\mathcal{C}\cup\{v_0\}$. Consider a weighted directed complete graph $G_0=(V,2\binom{V}{2})$, whose edge lengths are defined by $\{c_{f}(u,v) \cdot d(u,v)\}_{u,v}$. Then, for a pair of nodes $u,v \in G_0$, $\widehat{d}(u,v)$ quantifies the distance of the shortest feasible path from $u$ to $v$ assuming full charge at $u$ and possibility of passing through charging stations.
{\sf SDFP} specializes to the {\it Tour Gas Station problem} studied in \cite{KMM2011gsp}, which is NP-hard. 

Note that we assume in {\sf SDFP} that the SoC is brought to its maximum at each charging station. Once we obtain a tour under this assumption, it can be turned into a flight mission plan with the minimal charging requirements using Lemma~\ref{lem:red}.

For $u\in V$, let $d_u'\triangleq\min_{v\in\mathcal{C}}\widehat d(u,v)$ be the distance to the nearest charging station from $u$, and  ${\tt s}_u'\triangleq\text{argmin}_{v\in\mathcal{C}}\widehat d(u,v)$ be the corresponding nearest charging station from $u$. Also, let $d_u''\triangleq\min_{v\in\mathcal{C}}\widehat d(v, u)$ be the shortest distance starting from a charging station to $u$, and  ${\tt s}_u''\triangleq\text{argmin}_{v\in\mathcal{C}}\widehat d(v, u)$ be the corresponding charging station. Since we have asymmetric distances, $s'_u$ is not necessarily equal to $s''_u$.

Following \cite{KMM2011gsp}, we make a mild assumption that for every $u\in\mathcal{S}\setminus \{v_0\}$ there is $v\in\mathcal{C}$ such that $d(u,v)\le\alpha \frac{U}{2}$, where $\alpha\in[0,1)$ and $U\triangleq \frac{\overline{B}-\underline{B}}{\eta_{\rm d}}$. Intuitively, this assumption indicates that the distance between a pair of sites is always reachable by the available battery capacity. This assumption can be justified (for $\alpha=1$) as follows.
For a location $u\in\mathcal{S}\setminus \{v_0\}$, if every $v\in\mathcal{C}$ is at a distance greater than $\frac{U}{2}$, then it is infeasible to visit $u$ without incurring the battery level below $\underline{B}$ (as the SoC drops below $\overline{B}-\eta_{\rm d} U=\underline{B}$).   
Moreover, $\beta$ is defined as the maximum ratio between edges going between two gas stations, such as $d(u,v) \leq \beta \cdot d(v,u) \quad \forall~u,v\in D$.

In the following, we present an algorithm to {\sf SDFP} and then {\sf DFP}. The main algorithm is ${\tt Find\mbox{-}plan}\big[V,d\big]$, which is an extension of the Hungarian Algorithm~\cite{frieze1982worst} for finding a tour for ATSP. It converts the graph into a bipartite graph, iteratively solves the minimum assignment problem, and finally combines all sub-tours to get a single tour. The resulting tour is passed to the procedure ${\tt Fix\mbox{-}plan}$ for converting it to a feasible flight mission plan $\mathcal{F}$, which might use a non-optimal charging strategy $b(\cdot)$. Then, the resulting plan $(\mathcal{F},b(\cdot))$ is delegated to procedure ${\tt Fix\mbox{-}charge}$ for finding the minimal charging requirements with respect to the flight mission plan $\mathcal{F}$.  
Specifically, the three procedures in  ${\tt Find\mbox{-}plan}\big[V,d\big]$  are:

\begin{itemize}
	
	\item ${\tt Init\mbox{-}distances}\big[V,\widehat{d},u,v\big]$: This provides a lower bound for an optimal solution. Namely, it finds for every pair of locations $u,v\in V$, the minimum possible distance $\widetilde{d}(u,v)$, and the corresponding shortest path $\cP(u,v)$ to go from $u$ to $v$ without going out of the operational range of the battery. Note that if $\widehat{d}(u,v) \leq  U-d_u''-d_v'$ then the drone can always go from $u$ to $v$ directly\footnote{That is, starting with SoC$=\overline B$ at ${\tt s}_u$, then the drone reaches $u$ with SoC $\overline B-\eta_{\rm d}d_u''$, and then it flies directly from $u$ to $v$ causing the SoC to drop to $\overline B-\eta_{\rm d}(d_u''+\widehat{d}(u,v))=\underline B+\eta_{\rm d}(U-d_u''-\widehat{d}(u,v))\ge \underline B+\eta_{\rm d}d_v'$ at $v$. Thus, there is sufficient battery at $v$ to reach ${\tt s}_v$.}. Otherwise, in the best possible scenario, the drone can reach $u$ with SoC at most $\overline B-\eta_{\rm d}d_u''$, next it can visit a sequence of charging stations (only if the distance $\widehat{d}$ between two  successive such stations is at most $U$), then, from the last station, it has to reach $v$ such that the SoC at $v$ is at least $\underline B+\eta_{\rm d}d_v'$ (so that there is sufficient battery to reach ${\tt s}_v$). In particular, the distance from $u$ to the first charging station on this path should be at most $U-d_u''$. Similarly, the distance from the last station on the path to $v$ should be at most $U-d_v'$. This explains the definition of the graph $G$ in line~\ref{ss1} of the procedure.    	
\end{itemize}

\begin{algorithm}[!t]
	\footnotesize
	\caption{${\tt Find\mbox{-}plan}\big[V,d\big]$}	
	\begin{algorithmic}[1]
		\State Compute pairwise shortest distances $\{\widehat{d}(u,v)\}_{u,v}$ on weighted directed graph $G_0=(V,2\binom{V}{2})$
		\For{each $u,v\in V$}
		\State $(\widetilde d(u,v)$, $\cP(u,v)) \leftarrow {\tt Init\mbox{-}distances}\big[V,\widehat{d},u,v\big]$
		\EndFor
		\State Consider the weighted directed graph $G=(V,E;\widetilde d)$ where $E=2\binom{V}{2}$
		\State $\mathcal{F}_0 \leftarrow$ find a tour using the Hungarian algorithm on $G$
		\State $\mathcal{F}\leftarrow$ {\tt Fix\mbox{-}plan}$[G,\mathcal{F}_0]$ 
		\State $b'(\cdot)\leftarrow$ {\tt Fix\mbox{-}charge}$[\mathcal{F},b(\cdot)]$ 
		\State \Return {($\mathcal{F},b'(\cdot)$)}			
	\end{algorithmic} 
	\label{alg:plner}
\end{algorithm}
\begin{algorithm}[!t]
	\footnotesize
	\caption{${\tt Init\mbox{-}distances}\big[V,\widehat{d},u,v\big]$} 
	\begin{algorithmic}[1]
		\If{$\widehat{d}(u,v) \leq U-d_u''-d_v'$}
		\State $\widetilde{d}(u,v)\leftarrow \widehat{d}(u,v)$, \quad
		$\cP(u,v)\leftarrow \{(u,v)\}$
		\State \Return {$(\widetilde d(u,v)$, $\cP(u,v))$}
		\Else
		\parState {Construct a weighted directed graph  $G=(\mathcal{C}\cup\{u,v\},E;w)$ where \label{ss1}\\
			$E\triangleq\big\{\{u,z\}:~z\in\mathcal{C},~\widehat{d}(u,z)\le U-d_u''\big\} \bigcup$ $\big\{\{v,z\}:~z\in\mathcal{C},~\widehat{d}(v,z)\le U-d_v'\big\} \bigcup$  $\big\{\{z,z'\}:~z,z'\in\mathcal{C},~\widehat{d}(z,z')\le U\big\}$\\
			and $w(z,z') \triangleq \widehat{d}(z,z')$ for all $z,z'\in \mathcal{C}\cup\{u,v\}$ }
		\parState {$\cP(u,v) \leftarrow $ shortest path between $u$ and $v$ in $G$ (with a set of edge lengths $\{w(u,v)\}_{u,v}$) }
		\State $\widetilde d(u,v) \leftarrow $ length of $\cP(u,v)$
		\State \Return {$(\widetilde d(u,v)$, $\cP(u,v))$} 
		\EndIf
	\end{algorithmic} 
\end{algorithm}
\begin{algorithm}[!t]
	\footnotesize
	\caption{${\tt Fix\mbox{-}plan}\big[G,\mathcal{F}_0\big]$} 	
	\begin{algorithmic}[1]
		
		\State $\mathcal{F}\leftarrow \emptyset$
		\For{each $(u,v)$ in $\mathcal{F}_0$}   
		\State Add $\cP(u,v)$ to $\mathcal{F}$
		\EndFor  
		\State Add to $\mathcal{F}$ a set of sub-tours $\left\{\{(u,{\tt s'}_u),({\tt s'}_u,{\tt s''}_u),({\tt s''}_u,u) \}:~u\in V\right\}$
		\For{$u\in V$}
		\If {$\mathcal{F}\setminus\{(u,{\tt s'}_u),({\tt s'}_u,{\tt s''}_u),({\tt s''}_u,u)\}$ is feasible} 
		\State $\mathcal{F}\leftarrow \mathcal{F}\setminus\{(u,{\tt s'}_u),({\tt s'}_u,{\tt s''}_u),({\tt s''}_u,u)\}$
		\EndIf
		\EndFor 
		\State \Return {$\mathcal{F}$}
	\end{algorithmic} 
	\label{alg:fix-tour}
\end{algorithm}
\begin{algorithm}[!t]
	\footnotesize
	\caption{${\tt Fix\mbox{-}charge}\big[\mathcal{F},b(\cdot)\big]$} 
	\begin{algorithmic}[1]
		
		\State Let $\mathcal{F}_{i_1},\ldots, \mathcal{F}_{i_r}$ be the charging stations, in the order they appear on $\mathcal{F}$
		\For {$j=0,1,\ldots,r$} 
		\State \mbox{$\displaystyle D_j= \eta_{\rm d}\sum_{k=i_{j}}^{i_{j+1}-1} c_{f}(\mathcal{F}_k,\mathcal{F}_{k+1})d(\mathcal{F}_k,\mathcal{F}_{k+1})$}
		\EndFor   
		\For {$j=1,\ldots,r$}
		\State $B_j\triangleq \eta_{\rm c}\sum_{k=1}^{j} b(\mathcal{F}_{i_k})$
		\EndFor
		\For{$j=r$ {\bf downto} $1$}
		\State $\displaystyle b'(\mathcal{F}_{i_j})=\max\{0,\frac{1}{\eta_{\rm c}}(\underline{B} - x_0 \mbox{+} \sum_{k=0}^rD_k - \sum_{k=1}^{j-1}B_k)\}$ 
		\If {$b'(\mathcal{F}_{i_j})>0$}
		\State {\bf exit}
		\EndIf
		\EndFor 
		\State \Return {$b'(\cdot)$}						
	\end{algorithmic} 
	\label{alg:fix-charge}
\end{algorithm}

\begin{itemize}
	\item ${\tt Fix\mbox{-}plan}\big[G,\mathcal{F}_0\big]$: Given the flight mission plan $\mathcal{F}_0$ returned by the Hungarian algorithm considering the weights $\widetilde{d}$, this procedure reconstructs a feasible flight mission plan $\mathcal{F}$ for {\sf SDFP}. First, each edge $(u,v)$ is replaced in the  flight mission plan by the corresponding path $\cP(u,v)$. Since the resulting mission plan may still be infeasible, the procedure adds to every site a round trip to the closest charging stations (from and to the site). Finally, the added stations are dropped sequentially in a greedy manner as long as feasibility is maintained. 
	
	\item ${\tt Fix\mbox{-}charge}\big[\mathcal{F},b(\cdot)\big]$: Starting from the flight mission plan $(\mathcal{F},b(\cdot))$ returned by ${\tt Fix\mbox{-}plan}\big[G,\mathcal{F}_0\big]$, this procedure finds a minimal amount of recharging energy, according to Lemma~\ref{lem:red}.   
	
\end{itemize}

Let ${\sf OPT}_{\sf DFP}$ and ${\sf OPT}_{\sf SDFP}$ be the optimal solutions of {\sf DFP} and {\sf SDFP}, respectively.

\bigskip
\begin{lemma}[ \cite{sundar2013algorithms}]
	\label{lem:GS}
	The flight mission plan $\mathcal{F}$ returned by algorithm ${\tt Find\mbox{-}plan}\big[V,d]$ has cost $\widehat{d}(\mathcal{F})\le \left(\frac{(1+\alpha+\alpha\beta) \log(|T|))}{1-\alpha}\right) {\sf OPT}_{\sf SDFP}$. 
\end{lemma}

The below theorem establishes that algorithm ${\tt Find\mbox{-}plan}\big[V,d]$ has an asymptotic constant-factor approximation guarantee for {\sf DFP}. 

\bigskip
\begin{theorem}\label{thm:main}
	The flight mission plan $(\mathcal{F},b'(\cdot))$ returned by algorithm ${\tt Find\mbox{-}plan}\big[V,d]$ has cost $$\tau(\mathcal{F})+\tau_{\rm c}(b'(\mathcal{F}))=O({\sf OPT}_{\sf DFP})+O(1).$$
\end{theorem}

\begin{proof}
	See the supplementary materials.
\end{proof}


\subsection{Practical Adaptations}\label{sec:ext}

The preceding section explored a basic setting of {\sc DFP} and its efficient algorithms. In practice, an automated drone management system may require more sophisticated options. Here, we provide two extensions to the above algorithms to obtain heuristics for further practical applications.
\begin{enumerate}
	\item {\bf Wind Uncertainty:} Under steady wind conditions, we assume in the preceding algorithms that $c_{f}(u, v)$ stays constant on the designated path from $u$ to $v$. To account for wind volatility, it should be more precisely represented by $c_{f}(u, v, w)$, where $w$ is a vector whose value lies in an uncertain domain $w \in W$. For example, $W$ can be defined by the anticipated speed and orientation ranges $[\underline{|w|}, \overline{|w|}]$, $[\underline{\theta}_w, \overline{\theta}_w]$. To adjust the algorithms, one can proceed conservatively by replacing $c_{f}(u,v)$ with $\overline{c}_{f}(u, v) = \max_{w \in W}  c_{f}(u, v, w)$.
	\item {\bf Variable Drone Speed:} During monitoring or patrolling missions, the drone may need to vary its speed uniformly at certain designated paths in $V$. In this case, we run the algorithms sequentially in multiple rounds, with an increasing drone speed at each round, until reaching infeasibility (higher speed may result in insufficient battery to reach some sites, hence resulting no feasible solution). Then we will enumerate all the optimal solutions across the rounds to find the best solution with the lowest total flight time. By enumerating the possibilities of different drone speeds, the algorithms can identify an optimal flight mission plan.
\end{enumerate}

%% file: sim.tex
\section{Experimental Validation and Case Studies}

As one demonstration, we implement the proposed approach in an automated drone management system and verify the produced flight mission plans and recharging strategies in a real-world experiment. Additionally, to complement the analytic results derived in Sec.~\ref{sec:ndl}, the average-case performance and scalability of the featured planner are scrutinized extensively through numerical simulations and diverse case studies.

\subsection{Drone Management System}
The system interface, depicted in Fig~\ref{fig:app}, allows users to specify individual goals and visualize the computed flight mission plan. The system connects to a cloud server, which accepts location data from the users and computes the optimal flight mission plan. Then, the drone is programmed to follow the pre-computed flight mission plan. 
For dynamic tracking, the data from onboard sensors, including GPS, video feed, and SoC can be fetched continuously to monitor the real-time flight status of the drone. Should notably deviant sensor measurements be detected from the values estimated in the pre-computed mission plan, real-time calibration can be performed to find the minimum adjustment to the previous plan.

\begin{figure}[h!]
	\centering 
	\includegraphics[width=.75\linewidth]{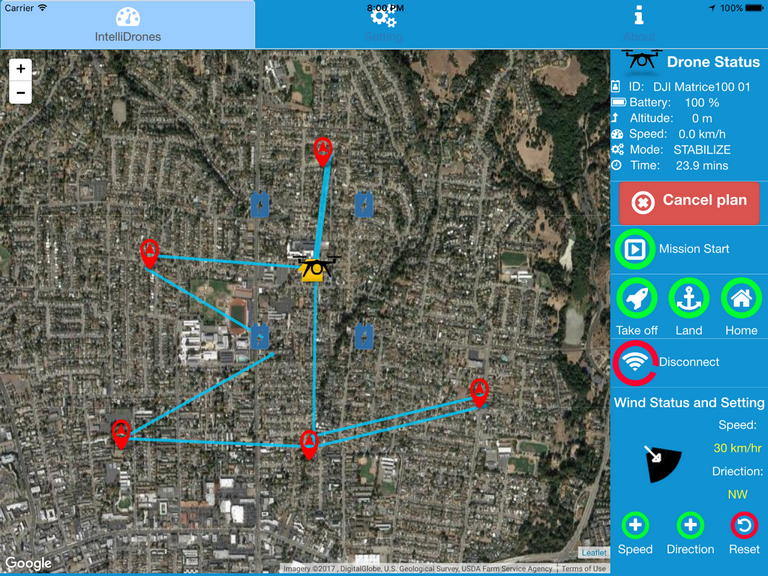}
	\caption{User interface of drone management system. }
	\label{fig:app}
\end{figure}

\subsection{Field Experiment}\label{exp}

To validate the practicality of the proposed planner and corroborate the power consumption model developed in Sec.~\ref{sec:battery}, a long-distance test mission was carried out in a real-world uncontrolled environment. Particularly, DJI Matrice 600 was deployed to patrol a rural area, delimited in Fig.~\ref{fig:drone_experiment_map}, with $4$ sites of interest and a charging station. The drone, which has a maximum flight duration of 38 minutes and a payload capacity of 6 kg, was loaded with a 2 kg weight and programmed to fly at 120-meter altitude maintaining 30 km/h horizontal speed. The lower bound on the battery SoC was set to 40\%. The average wind speed was estimated to 30 km/h, as measured by a portable wind meter at the start- and end-points of the route. Video footage of the experiment is available online\footnote{\url{https://www.dropbox.com/s/mez5hnyg7czkwyf/drone_tsp_nk_50mb.mp4?dl=0}}.

\begin{figure}[!t]
\center
\captionsetup[subfigure]{labelformat=empty}
\subfloat[]{  \includegraphics[width=0.9\linewidth]{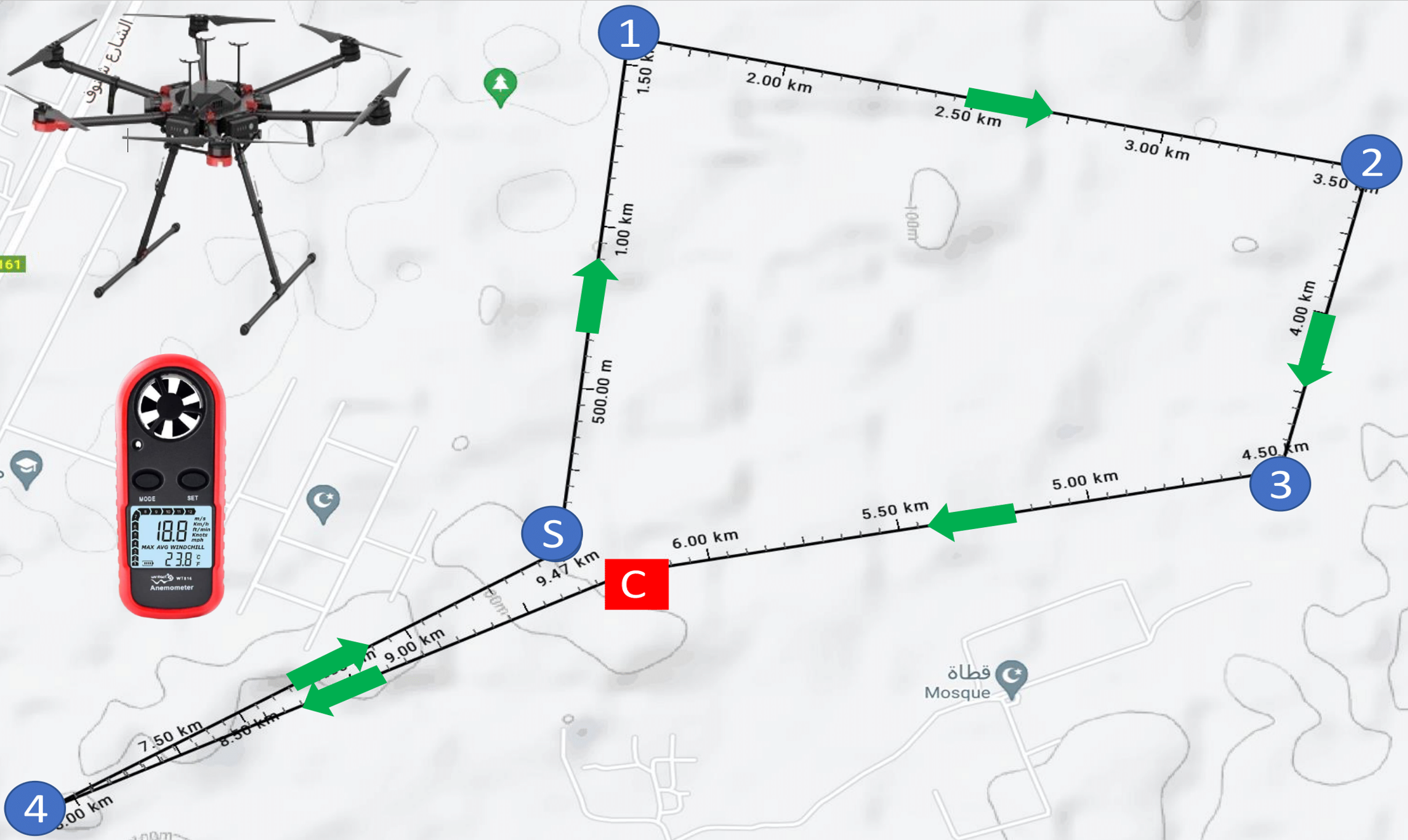}}\vspace{-7mm}

\subfloat[]{  \includegraphics[width=.915\linewidth]{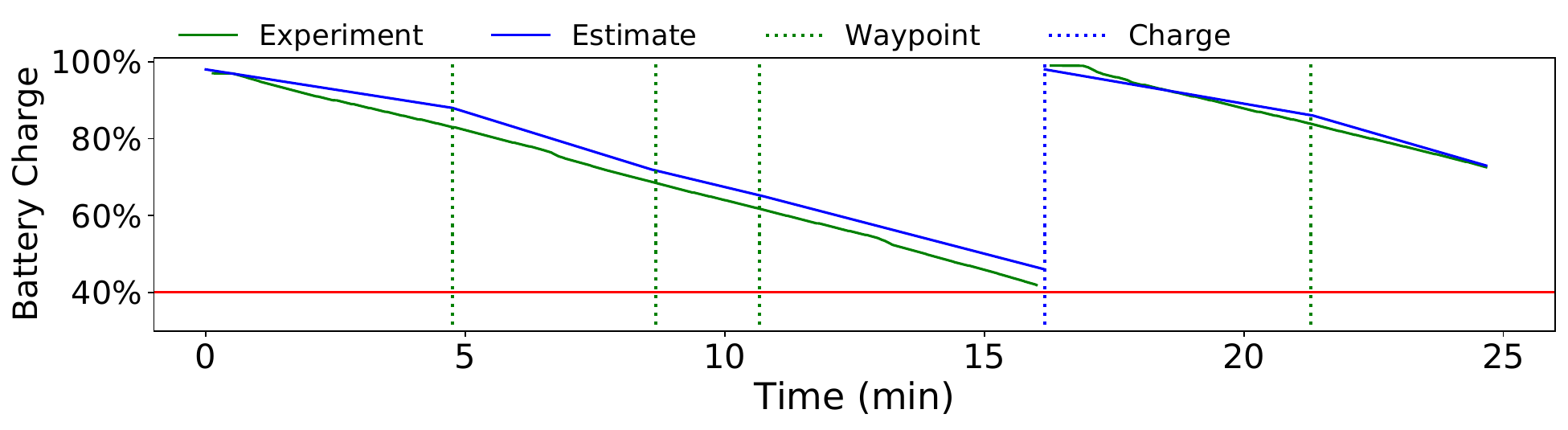}}
\caption{The setup and results of the long-distance mission experiment with DJI Matrice 600. The illustration on top captures the map of the area and locations of the target sites, where S denotes the starting point and C locates the charging station. The bottom plot collates the measured and predicted battery consumption of DJI Matrice 600, with vertical green dashed lines representing the target sites and the dashed blue line representing the charging station.}
\label{fig:drone_experiment_map}
\end{figure}

As displayed in Fig.~\ref{fig:drone_experiment_map}, the proposed planner in Alg.~\ref{alg:fix-tour} returned the flight mission route [S, 1, 2, 3, C, 4, S], which has a total distance of 9.47 km and includes one recharging stop. The detour to C resulted from the selected minimum battery SoC limit of 40\% (pictured as a red line in Fig.~\ref{fig:drone_experiment_map}). From the experiment results, we observe a slight discrepancy (of around 5\%) between the actual battery energy levels and the estimates provided by the regression model developed in Sec.~\ref{sec:battery}. This is attributed primarily to inaccurate wind profile as the measurements were taken only at one location.

\subsection{Case Studies}

\noindent{\bf Setup:} The studied scenario considers four sites of interest and four charging stations. Fig.~\ref{fig:simsetup} depicts the locations of the sites (as black points), charging stations (as blue squares) and the base (as magenta triangle) wherefrom the test drone (3DR Solo) departs for the mission. Positioning of these nodes and their inter-distances are based on locations in a suburban community in Abu Dhabi.

The following two major sets of studies were undertaken.
\begin{itemize}

\item {\itshape  Study 1}: Here, 4 sub-cases were examined under different wind and payload conditions. In the first two, 3DR Solo is equipped only with one battery and the average wind speed is set to 18 km/h. Then we double the battery capacity with the same wind condition in the remaining two sub-cases. Since the battery capacity was doubled, extra weight was added to the drone. Table~\ref{tab:setuptable} summarizes the parameters of all the sub-cases, which are denoted by $\tt{S_1C_1}$ to $\tt{S_1C_4}$.
\item {\itshape Study 2}: To evaluate the planner's performance under uncertainties, we vary the wind speed and orientation within a certain range then investigate whether a feasible solution can be obtained. We select the sub-case with the shortest total trip time in Study 1 (i.e., $\tt{S_1C_1}$) and then gradually elevate the level of uncertainties. Specifically, four sub-cases, denoted by $\tt{S_2C_1}$ to $\tt{S_2C_4}$, are considered wherein the wind speed and orientation are varied from 9 to 21 km/h and from 0$^\circ$ to 360$^\circ$ in four discrete scales (e.g., $\tt{S_2C_1}$ assumes the wind speed could range from 9 to 12 km/h with an orientation lying within $[-45^\circ, 45^\circ]$).
\end{itemize}

\begin{figure}[!t]
	\hspace{5pt}
	\begin{minipage}[!htb]{0.53\linewidth}
		\centering
		{\includegraphics[width=.8\linewidth]{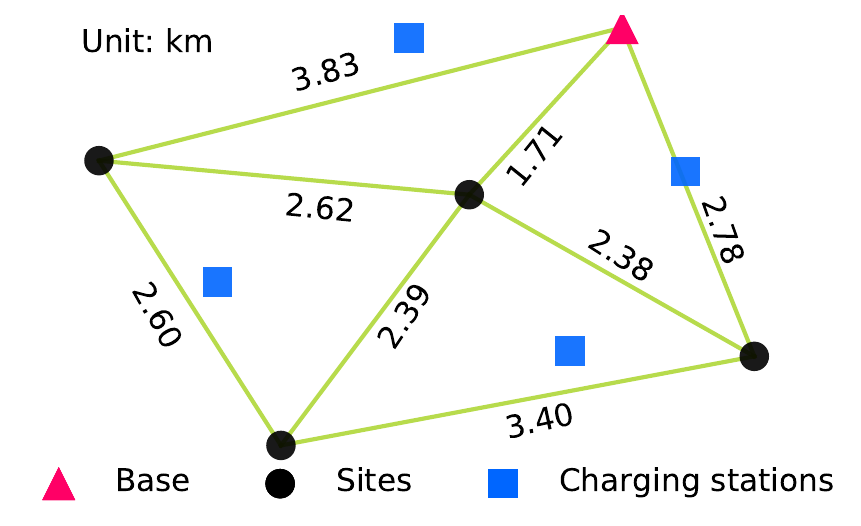}} 
		\captionof{figure}{Map of the simulated environment.}
		\label{fig:simsetup}
	\end{minipage}\hfill
	\begin{minipage}[h]{0.38\linewidth}
		\centering
		\scalebox{0.8}{\begin{tabular}{@{}c|@{}c@{}|@{}c@{}|c@{}} 
				\hline
				\hline
				Case & Battery & $\vec{w}_{xy}$ & $m$\\
				& (Wh) & & (g)\\
				\hline
				1 & 70 & South & 0\\
				2 & 70 & North-East & 0\\
				3 & 140 & South & 500\\
				4 & 140 & North-East & 500\\
				\hline
				\hline
		\end{tabular}}
		\captionof{table}{Parameters of Study 1.}
		\label{tab:setuptable}
	\end{minipage}
\end{figure}

\begin{figure}[!b]
	\center
	\captionsetup[subfigure]{labelformat=empty}
	\subfloat[]{  \includegraphics[clip,width=0.7\linewidth]{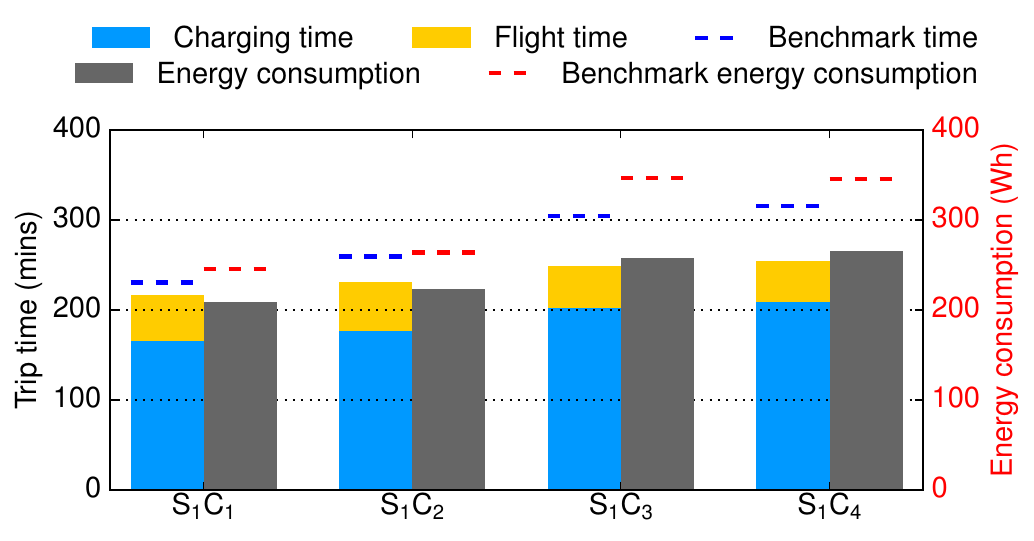}}\vspace*{-7.5mm}
	
	\subfloat[]{\hspace{4pt}\includegraphics[trim=0 0 0 19mm, clip,width=0.7\linewidth]{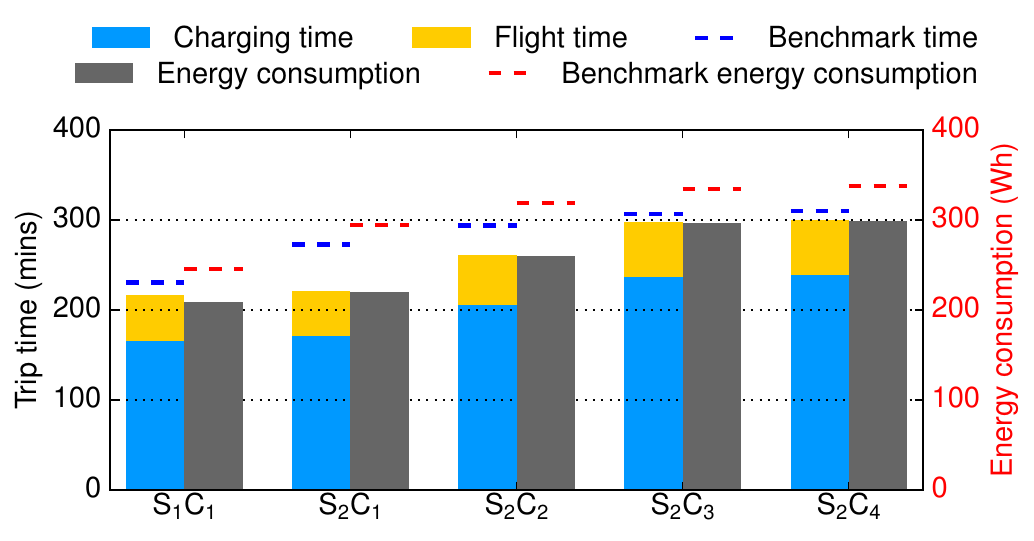}}
	\caption{Trip time and energy consumption of 3DR Solo in Study 1 and Study 2 on top and bottom, respectively. $\tt{S_1C_1}$ is appended to the bottom plot merely for comparison purposes.}
	\label{fig:simstat1-2}
\end{figure}

\begin{figure*}[!t]
\centering
\subfloat[Case 1]{\label{fig:eva1}\includegraphics[width=0.255\textwidth]{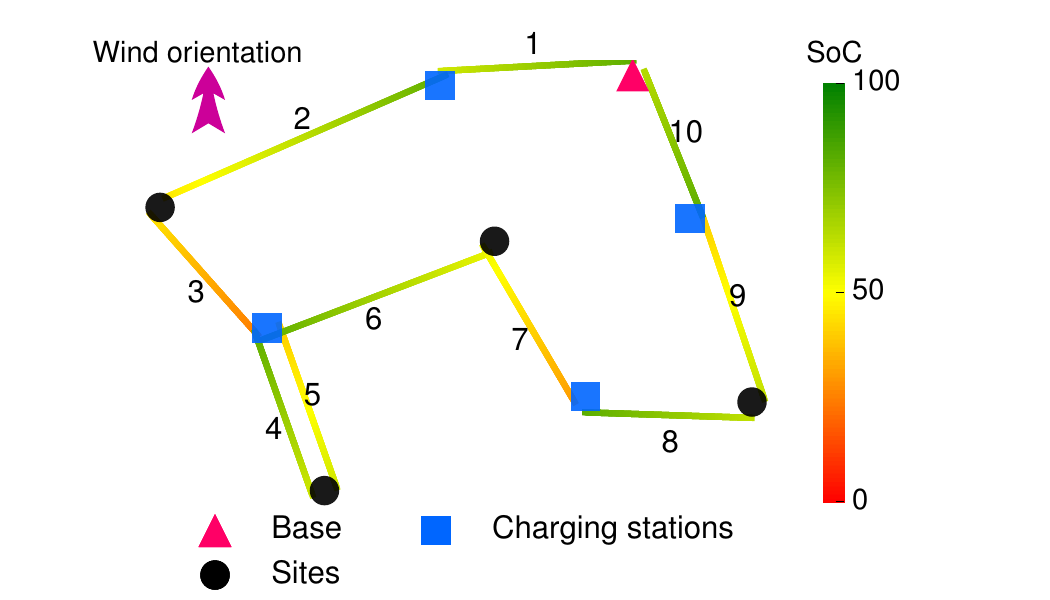}}
\subfloat[Case 2]{\label{fig:eva2}\includegraphics[width=0.255\textwidth]{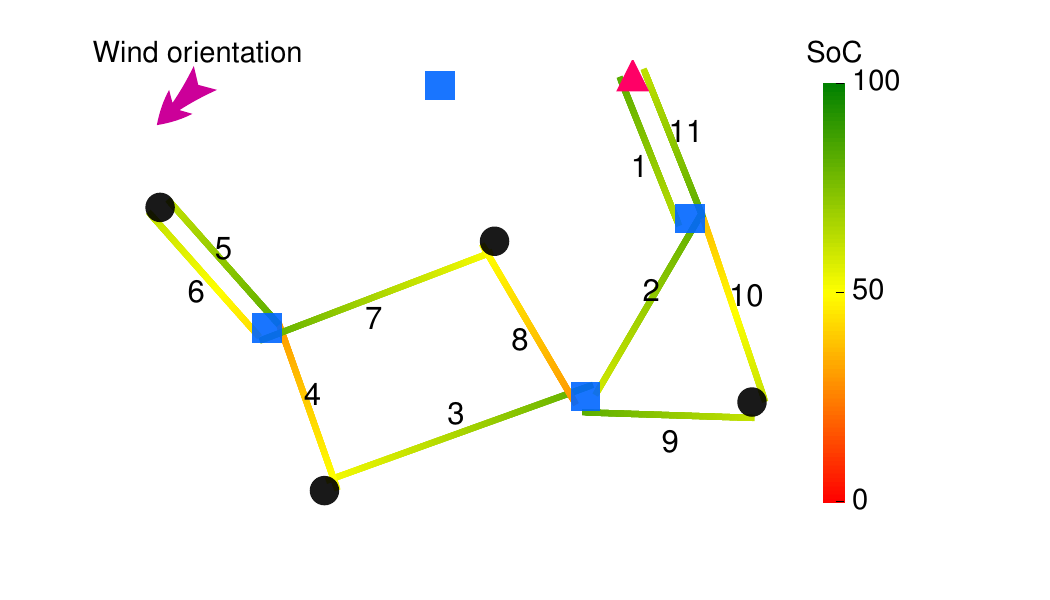}}
\subfloat[Case 3]{\label{fig:eva3}\includegraphics[width=0.255\textwidth]{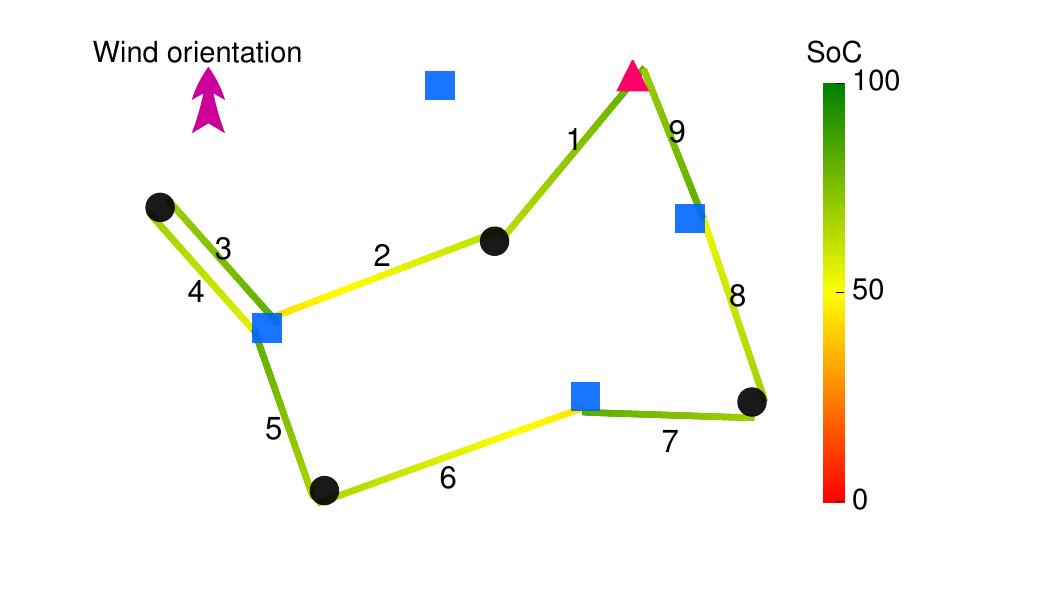}}
\subfloat[Case 4]{\label{fig:eva4}\includegraphics[width=0.255\textwidth]{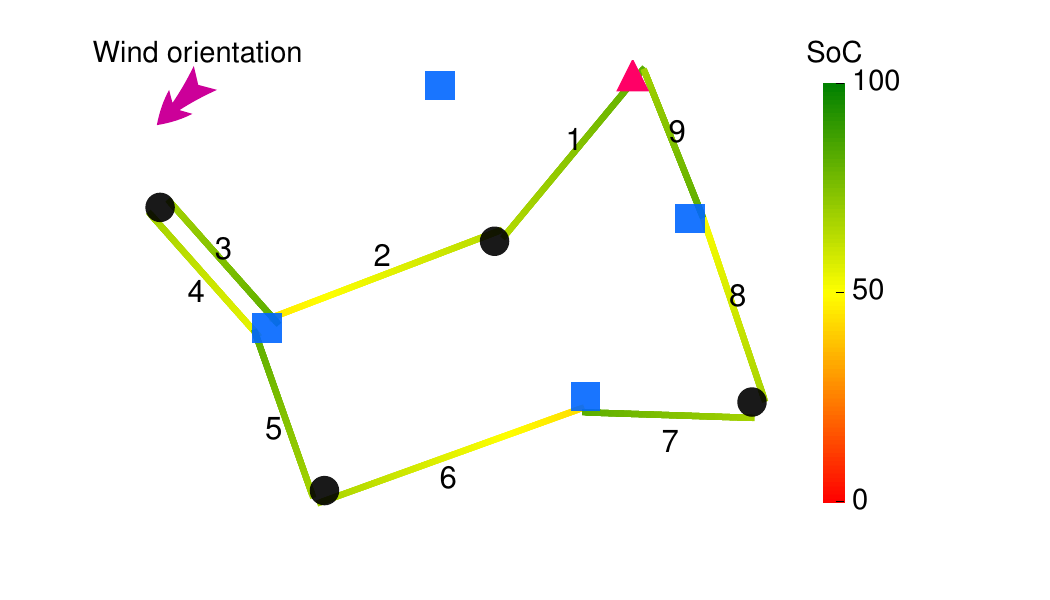}}

\subfloat[Case 1]{\label{fig:evaa}\includegraphics[width=0.22\textwidth]{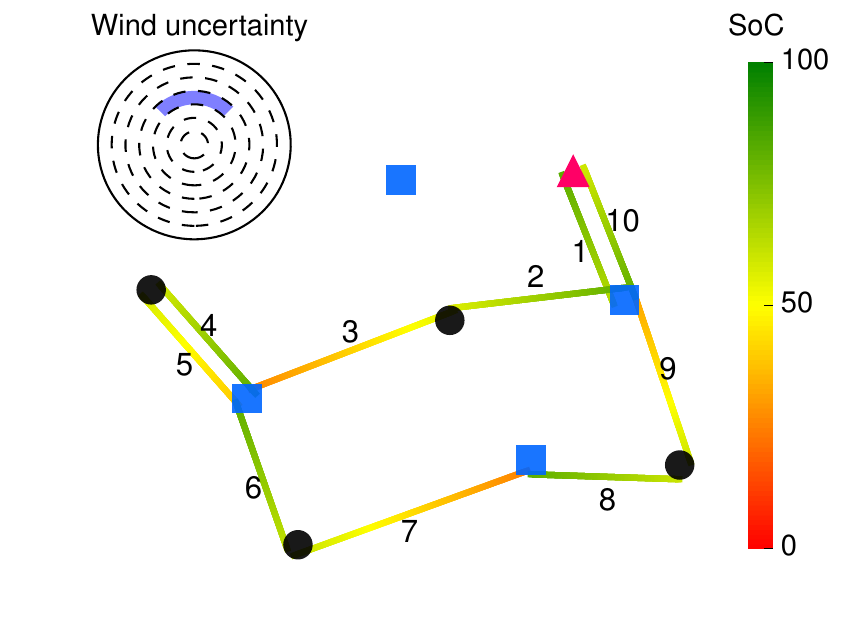}}\hspace{16pt}
\subfloat[Case 2]{\label{fig:evab}\includegraphics[width=0.22\textwidth]{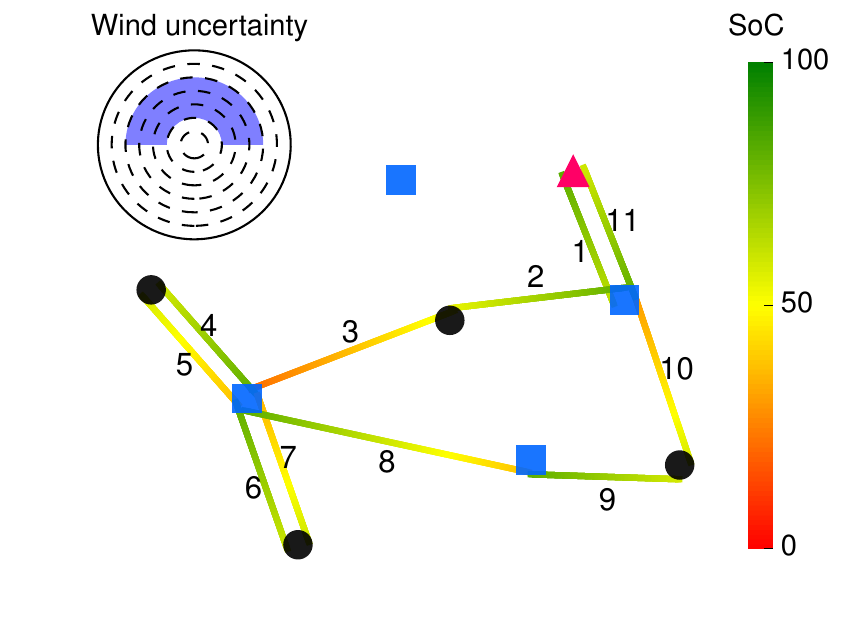}}\hspace{16pt}
\subfloat[Case 3]{\label{fig:evac}\includegraphics[width=0.22\textwidth]{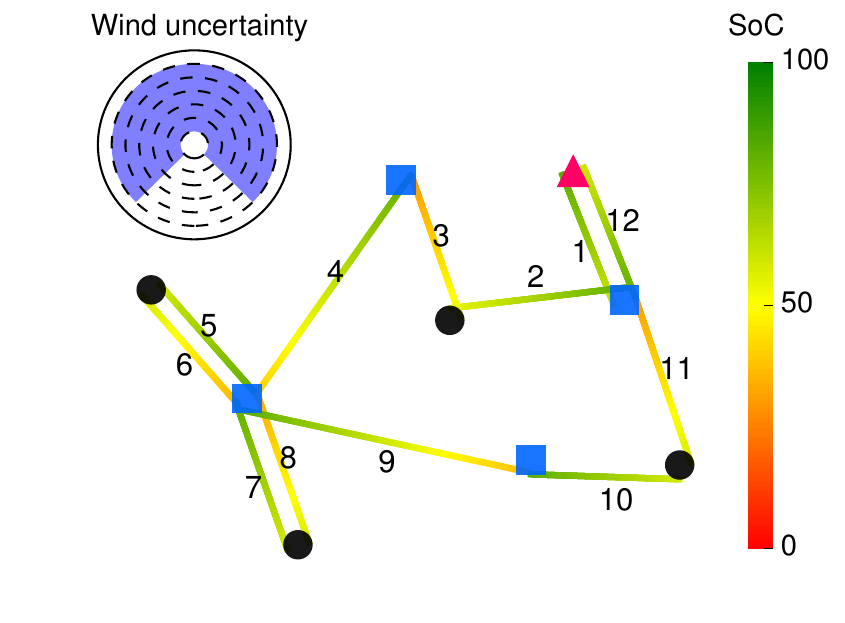}}\hspace{10pt}
\subfloat[Case 4]{\label{fig:evad}\includegraphics[width=0.22\textwidth]{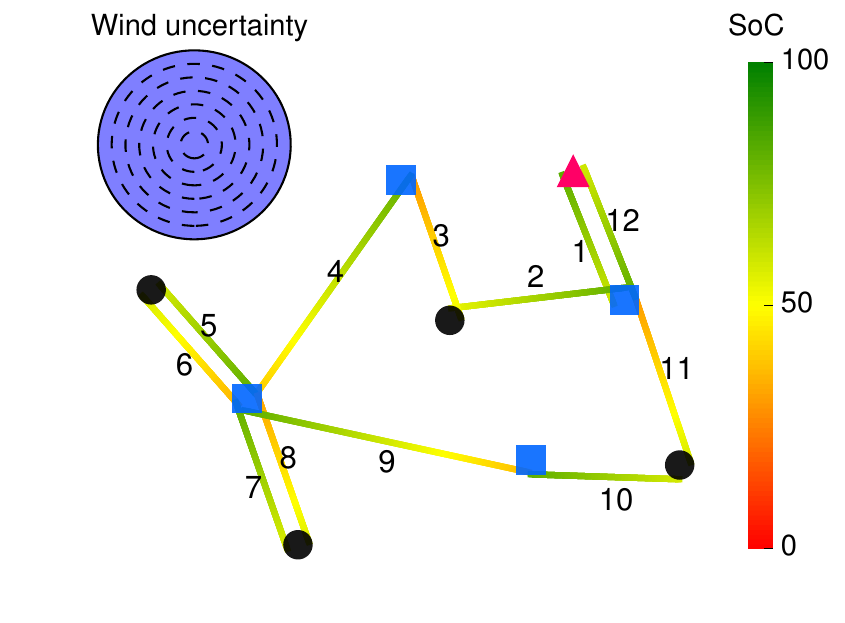}}
 \caption{Visualized results of Study 1 (a-d) and Study 2 (e-h).}
\label{fig:sim1-2}
\end{figure*}

\noindent{\bf Results:} For comparison, we employ a {\em benchmark} strategy that routes the drone to the nearest unvisited site, and flies to a charging station if the battery SoC drops below a preset threshold. The minimum SoC that could fly to a nearest charging station from any site serves as the threshold.

Fig.~\ref{fig:sim1-2} visualizes the flight mission plans produced by the proposed approach for both case studies, while Fig.~\ref{fig:simstat1-2} draws a comparison against the benchmark algorithm in terms of the travel time and energy consumption. In Fig.~\ref{fig:sim1-2}, the numbers indicate the path order of the drone, colors represent the battery SoC, wind speeds and orientations are pictured on the upper-left corners.

\noindent{\itshape Study 1}: 
The following two findings transpire: (1) The north-east wind lead to higher energy consumption than the south wind. Besides, longer travel time is observed due to increased charging duration. (2) Increasing battery capacity of 3DR Solo did not induce reduced travel time. We observe that though the flying time in $\tt{S_1C_3}$ is the shortest, it takes more time to charge since the drone becomes heavier by carrying extra weight for the battery, consequently resulting in longer total travel time.

\noindent{\itshape Study 2}: As deduced from Fig.~\ref{fig:simstat1-2}, the energy consumption increases with the rising level of uncertainties. To that extent, the worst scenario is captured by $\tt{S_2C_4}$, in which the drone may always fly into a tailwind. Thus, $\tt{S_2C_4}$ provides the most conservative result. Also, for high uncertainty levels, the observed performance gap (w.r.t. the trip time and energy consumption) between the proposed approach and the benchmark diminished. This is due to the latter taking more frequent recharging decisions in provision for heightened uncertainties.

\subsection{Numerical Analyses}
To conclude the evaluation, we investigate the proposed approach's average-case performance and scalability via simulations on randomly generated large-scale instances with up to $200$ vertices. Therein, the number of charging stations is set to $\frac{1}{5}$th the number of vertices, which are positioned uniformly at random on a canvas area of 3.33 km by 3.33 km. The drone speed is fixed to 18 km/h and the battery capacity to $80$ Wh. Wind speed is sampled uniformly at random from $[0, 3.6]$ km/h with a direction chosen uniformly from $[215^\circ, 235^\circ]$. 

Fig.~\ref{fig:ATSP_approx} plots the {\em empirical} and {\em theoretical} approximation ratios attained by Alg.~\ref{alg:fix-tour} for {\sf SDFP} across 10 runs at 95\% confidence interval. The optimal solutions of {\sf SDFP} were computed by a numerical solver based on the MILP formulation provided in the supplementary materials. In the simulations, Alg.~\ref{alg:fix-tour} exhibited favorable average-case performance, far surpassing its theoretical worst-case guarantees stated in Lemma~\ref{lem:GS}. Furthermore, as the number of vertices grew, the empirical approximation ratio of Alg.~\ref{alg:fix-tour} improved, approaching near-optimality.
\begin{figure}[!t]
	\subfloat[]{\label{fig:ATSP_approx}\includegraphics[width=.47\linewidth]{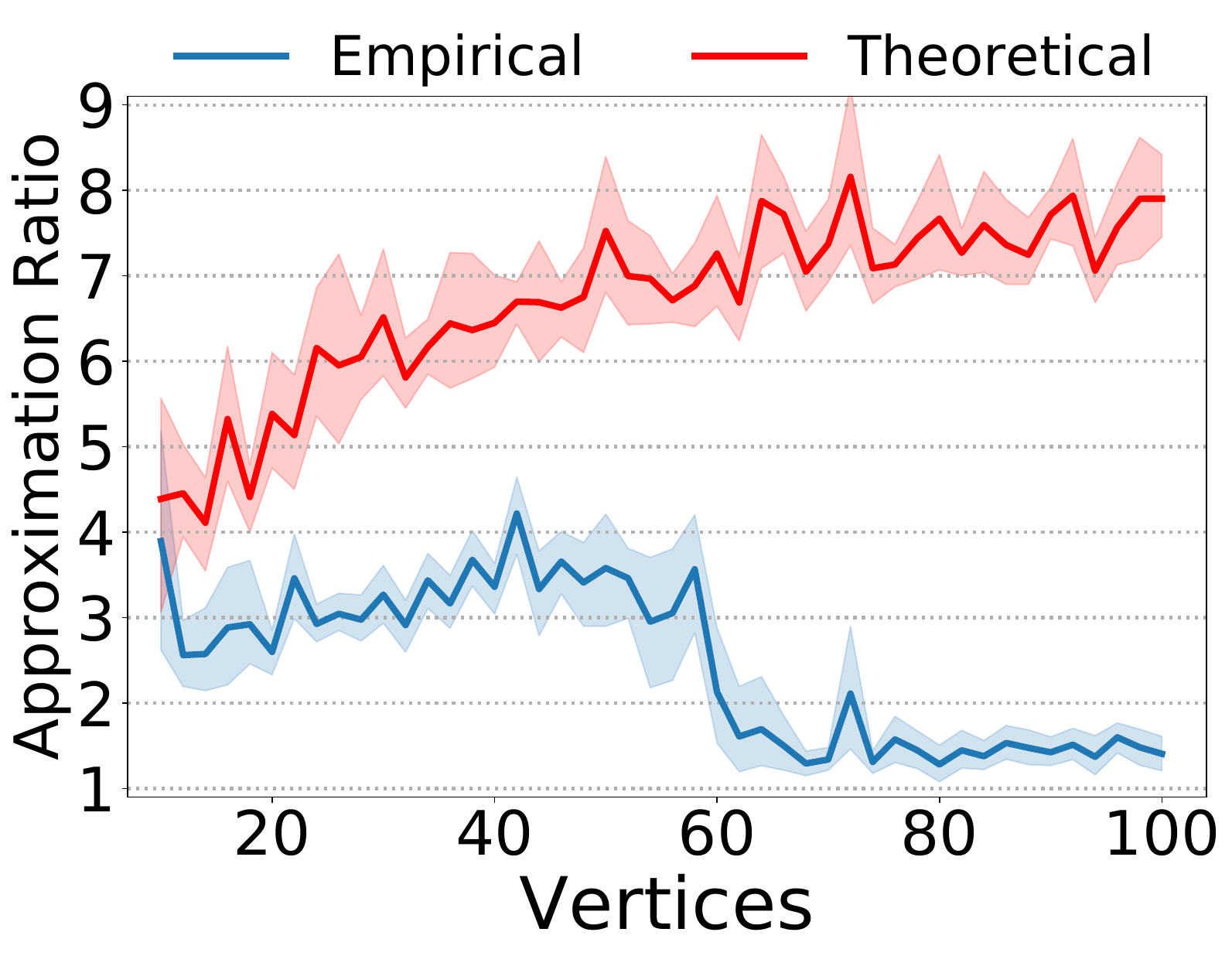}}\hspace{2pt}
	\subfloat[]{\label{fig:ATSP_gas_time}\includegraphics[width=.515\linewidth]{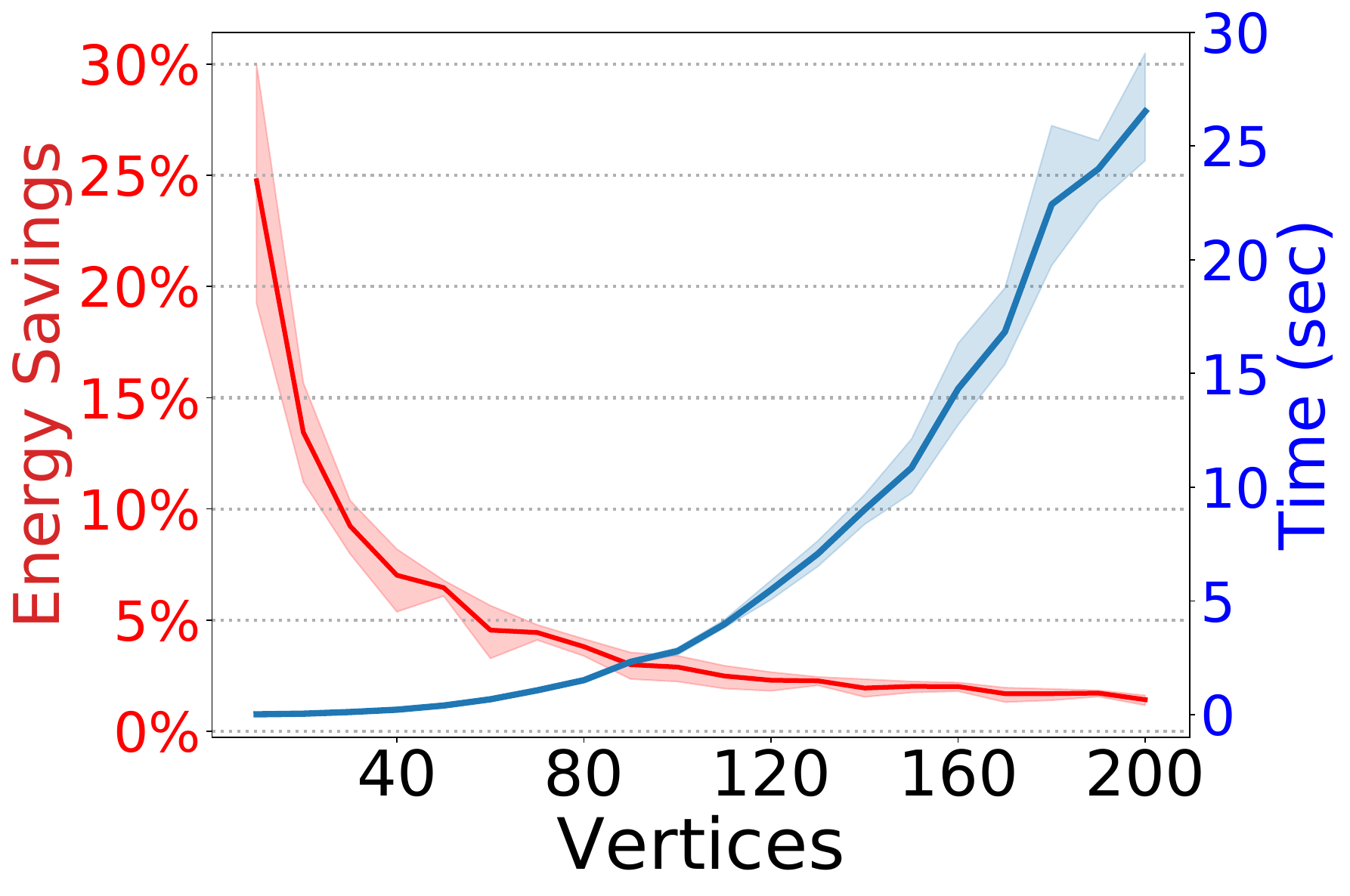}}
	\caption{ The algorithms' performance against the number of vertices at 95\% confidence interval: (a) Approximation ratio (b) Running time and recharging optimization gains.}
\end{figure}

Lastly, Fig.~\ref{fig:ATSP_gas_time} explores the average computational complexity (over $10$ runs) of the planner introduced in Alg.~\ref{alg:plner} (implemented in Python 3 on an Intel i9-9900k CPU) along with the total {\em energy savings} ensued from recharging optimization (Alg.~\ref{alg:fix-charge}). For up to $200$ vertices, the measured running time did not exceed $30$ seconds, which could be further reduced in a computationally more efficient programming environment (e.g., C++). With respect to recharging optimization, the achieved energy savings were more substantial for small-scale instances, in which the share of partial recharges was relatively more prominent.


%% file: app.tex
\section{Prospective Extensions}\label{exte}

\subsection{Fleet Mission Planning with Recharging Optimization}
		In extending the current problem setup to a more general setting with multiple UAVs, we differentiate between two scenarios, namely with shared (i.e., supports charging multiple drones simultaneously) and non-shared charging stations. For instance, a charging site with multiple adjacently placed inductive recharging landing pads (such as those mentioned in~\cite{chittoor2021review}) exemplifies the shared setup, whereas the one with a single inductive pad corresponds to the non-shared scenario. In the former case, one can avail the polynomial-time transformation proposed by Bellmore and Hong~\cite{bellmore1974transformation} to convert the problem to an equivalent instance of standard TSP on an expanded graph, which can then \textit{be directly tackled} by the algorithm proposed in Alg.~\ref{alg:plner} while retaining the approximation guarantee in Theorem~\ref{thm:main}. In the resultant solutions, several UAVs may be assigned to charge simultaneously at a single station, hence the requirement of shared charging stations. 
		
	    On the other hand, for the scenario where shared charging is not supported, the problem transforms into joint mission planning and \textit{recharging scheduling}. One sub-optimal approach is to further extend the adapted MILP formulation of {\sf SDFP} with additional constraints that cater for multi-UAV routing, as detailed in Sec.~\ref{multiuav} in the supplementary materials. As indicated by the simulations results reported in Sec. C in the supplementary materials, the augmented formulation can handle sufficiently large instances. However, since the formulation lacks representation of time, it allows a charging station to be used only once. The output solutions might yield arbitrarily worse objective values compared to the optimum (with recharging scheduling), thereby stimulating future research into developing efficient approximation algorithms for this problem.

\subsection{Mission Planning with Mobile Charging Stations}
        Another promising avenue to explore is the extension to the case with mobile charging stations. Such a setup allows for more flexible and robust drone management system that caters for uncertainties. In~\cite{majid17charger}, an autonomous robotic mobile platform (see Fig.~\ref{fig:charging}) was proposed that can recharge drones automatically. The system includes a tethered rover that is capable of autonomous navigation, and is equipped with a robotic arm carrying an inductive charging pad, which can flexibly recharge drones of varying sizes and shapes from different positions. With this dynamic model, the problem would additionally include optimizing the number of rovers and their routing.

	\begin{figure}[ht]
		\centering
		\includegraphics[width=.7\linewidth]{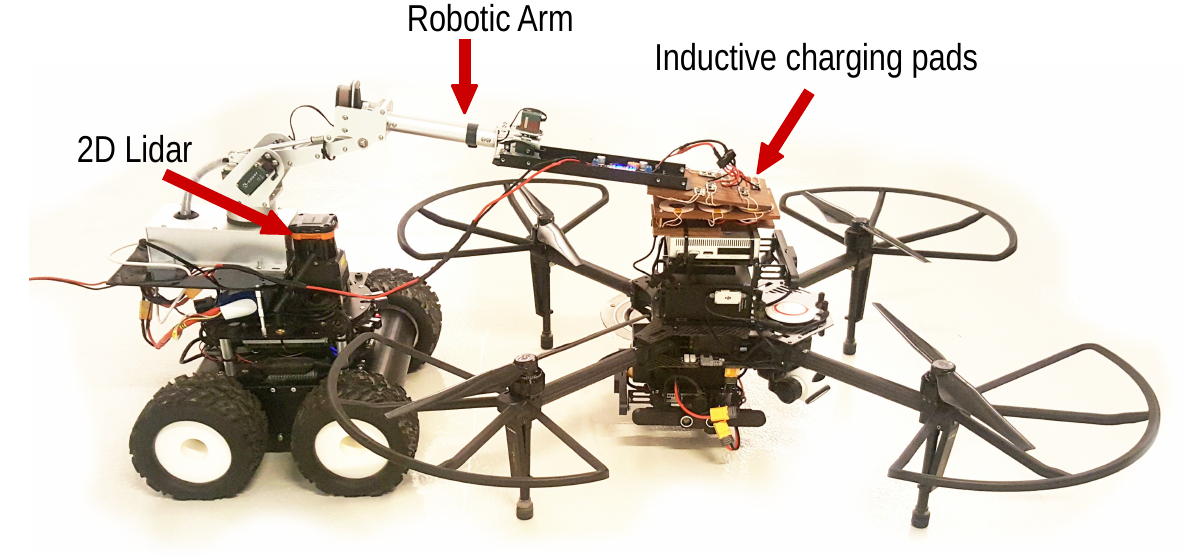}
		\caption{Robotic mobile charging station for drones.}
		\label{fig:charging}
	\end{figure}

\color{black}

%% file: charging.tex
%% file: concl.tex
\section{Conclusion}

To consolidate the practical applications of drones, this paper developed and experimentally verified an automated tour management system for an energy-constrained UAV deployed on long-distance flight missions (e.g., for monitoring or data acquisition purposes). Through extensive experimentation and analysis, we derived an effective power consumption estimation model for multi-copter UAVs and validated it on multiple drones. With this model, we formulated the energy-constrained tour management problem as a multi-objective extension of ATSP and developed an efficient mission planning algorithm with certifiable performance guarantees. Future work can be directed towards developing approximation algorithms for extended variants of the problem with multiple UAVs and non-shared/mobile charging stations.

%% file: appendix.tex
\setcounter{equation}{0}
\renewcommand{\theequation}{S.\arabic{equation}}
\appendix

\subsection{Details of Test Drones}

\begin{table}[!htb]
	\centering
	\resizebox{\linewidth}{!}{
	\begin{tabular}{@{}r|c|c|c} 
		\hline
		\hline
		& 3DR Solo & DJI Matrice 100 & DJI Matrice 600 \\
		\hline
		Weight &  2 kg & 2.8 kg & 9.1 kg\\
		Dimensions & 25cm $\times$ 46cm & 46cm $\times$ 46cm & 166.8cm $\times$ 151.8cm\\
		Battery & 5200 mAh 14.8V  & 5700 mAh 22.8V & 4500 mAh 22.2V ($\times$6)\\
		Battery Weight & 500 g & 600 g & 595 g($\times$6)\\
		Motors &  880 kV ($\times$4) & 350 kV ($\times$4) & 130 kV ($\times$6)\\
		Max Speed & 20 km/h & 60 km/h  & 65 km/h \\
		Max Altitude & 120 m (FAA Regulation) &  120 m (FAA Regulation) &  120m (FAA Regulation)\\
		Charging Duration & 90 mins & 180 mins & 92 mins\\
		Software & Python Dev Kit  & DJI SDK \& ROS & DJI SDK \& ROS\\
		\hline \hline
	\end{tabular}  
	}
	\caption{Specifications of test drones.}
	\label{tab:drone-solo}
\end{table}

\subsection{Mixed-Integer Linear Formulation of {\sf SDFP}}

This section provides the mixed-integer linear programming formulation of {\sf SDFP} adapted from~\cite{sundar2013algorithms}.  

Given a complete directed graph $G=(V,E)$ with $V$ as vertices and $E$ representing the set of edges, define an integer variable $x_{ij}$ which determines the number of edges from vertex $i$ to $j$ in $V$ (note that it may exceed one for edges between charging stations). Let $L=\overline{B}-\underline{B}$ be the available battery capacity and $\eta_{\rm d}\widehat{d}(i,j)$ be the energy consumption on edge $(i,j) \in E$. Denote by $p_{ij}$ the flow from $i$ to $j$, by $r_j$ the energy level after reaching vertex $j$, and by $v_0$ the base. With this notation, {\sf SDFP} can be recast into the following mixed-integer linear program:

\begin{align}
	&\min_{x,p,r} \sum_{(i,j)\in E} c_{ij} x_{ij} \notag\\
\text{subj} &\text{ect to:}\notag\\
	&\sum_{i\in V\backslash\{k\}} x_{ik} = \sum_{i\in V\backslash \{k\}} x_{ki} \quad \forall k \in V \label{eq1}\\
	&\sum_{i\in V\backslash \{k\}} x_{ik} =1 \quad \forall k \in \mathcal{S} \cup \{v_0\}, \label{eq2}\\
	&\sum_{i\in V} (p_{v_0i}-p_{iv_0}) =|\mathcal{S}| -1  \label{eq3} \\
	&\sum_{j\in V\backslash \{i\}} (p_{ji}-p_{ij}) =1 \quad \forall i\in \mathcal{S}\backslash v_0  \label{eq4} \\
	&\sum_{j\in V\backslash \{i\}} (p_{ji}-p_{ij}) =0 \quad \forall i\in \mathcal{C}  \label{eq5} \\
	&0 \leq p_{ij} \leq |\mathcal{S}| x_{ij} \quad \forall i,j \in V  \label{eq6}\\
&	r_j-r_i + \eta_{\rm d}\widehat{d}(i,j) \leq M(1-x_{ij}) \quad \forall i,j \in \mathcal{S}  \label{eq7} \\
&	r_j-r_i + \eta_{\rm d}\widehat{d}(i,j) \geq -M(1-x_{ij}) \quad \forall i,j \in \mathcal{S}  \label{eq8} \\
&	r_j-L + \eta_{\rm d}\widehat{d}(i,j) \geq -M(1-x_{ij}) \quad \forall i \in \mathcal{C}, j \in \mathcal{S}  \label{eq9} \\
&	r_j-L + \eta_{\rm d}\widehat{d}(i,j) \leq M(1-x_{ij}) \quad \forall i \in \mathcal{C}, j \in \mathcal{S}  \label{eq10} \\
&	r_i -\eta_{\rm d}\widehat{d}(i,j) \geq -M(1-x_{ij}) \quad \forall i \in \mathcal{S}, j \in \mathcal{C}  \label{eq11}\\
	& 0 \leq r_i \leq L \quad \forall i \in \mathcal{S} \label{eq12}\\
	& x_{ij} \in \{0,1\} \quad \forall i,j \in V \quad \text{(either $i$ or $j$ is a target)}\\
	& x_{ij} \in \{0,1,2,...,|\mathcal{S}|-1\} \quad \forall i,j \in \mathcal{C},
\end{align} 

\noindent where Constr.~\eqref{eq1} balances the in-degree and out-degree of each vertex and Constr.~\eqref{eq2} ensures that each site is visited once only. The equations in~\eqref{eq3}-\eqref{eq6} account for arc flow conservation, assuring that the $|\mathcal{S}|$ units of outbound flow from the depot are equally distributed across the target sites such that each receives exactly one unit of flow. The flight energy requirements are captured by Constrs.~\eqref{eq7}-\eqref{eq11}, where $M$ stands for a sufficiently large constant. For a UAV traveling from site $i$ to $j$, Constrs.~\eqref{eq7}-\eqref{eq8} guarantee that the energy level at $j$ is $r_j= r_i - \eta_{\rm d}\widehat{d}(i,j)$. Whereas for the flight from a charging station $i$ to a target site $j$, ~\eqref{eq9}-\eqref{eq10} enforce the battery charge at $j$ to be $r_j= L - \eta_{\rm d}\widehat{d}(i,j)$.



\subsection{Multi-UAV Mission Planning with Non-shared Charging Stations}\label{multiuav}
Here, we extend the MILP formulation of {\sf SDFP} to the setting with multiple UAVs and non-shared charging stations. Let $N$ ($N \leq  |\mathcal{S}|$) be the number of UAVs stationed at the base $v_0$. Then the problem can be defined as follows:
\begin{align}
	&\min_{x,p,r} \sum_{(i,j)\in E} c_{ij} x_{ij}\notag\\
\text{subject to} \quad&\eqref{eq1}, \eqref{eq3} -  \eqref{eq12}\notag\\
	&\sum_{i\in V\backslash \{k\}} x_{ik} =1 \quad \forall k \in \mathcal{S}\backslash v_0,\label{mu1}\\
    &\sum_{i\in V\backslash \{k\}} x_{ik} \leq 1 \quad \forall k \in \mathcal{C}, \label{mu2}\\
    &\sum_{i\in V\backslash \{v_0\}} x_{v_0i} = N ,\label{mu3}\\
	& x_{ij} \in \{0,1\} \quad \forall i,j \in V
\end{align} 

In the above program, Constr.~\eqref{mu1} states that $N$ number of edges will be selected from the starting node. Considering that each charging station is allowed to be utilized only once, as enforced by Constr.~\eqref{mu2}, the tours assigned to UAVs will have non-intersecting paths. Given that all the tours are disjoint, the energy constraints in~\eqref{eq7}-\eqref{eq12} become independent of the number of UAVs since they rely only on the visited target sites. Note that this formulation assumes uniform battery capacity $L$ for all the drones. Also, the above program stipulates dispatching exactly $N$ drones, however one can replace the equality in Constr.~\eqref{mu3} with inequalities to regulate the number of deployed drones. 


To investigate the scalability of the above MILP, we performed a set of simulations on a 100x100 2D surface. The sites of interest and charging stations were placed randomly (uniform selection) on the surface with a three-to-one ratio of sites to stations. The distance and fuel cost between nodes were computed based on the Euclidean metric, and each drone's batter capacity $L$ was set to $120$. The setup considers three different numbers of UAVs and a total number of vertices ranging from $20$ to $100$ in steps of $5$. The simulations were implemented with Python and Gurobi solver on an Intel i9-9900k CPU with 32 GB of RAM.

Fig.~\ref{fig:multi-agent-milp} depicts the computational time of the MILP program across $30$ runs at 99\% confidence interval and the average tour cost per UAV. As observed from the figure, for large-scale instances with $100$ vertices, the running time did not exceed $100$ seconds. Also, as expected, the average tour cost per UAV decreased with growing fleet size.

As an illustration of multi-UAV mission planning with non-shared charging stations, we invoke the above MILP program to emulate the field experiment reported in Sec.~\ref{exp} with two drones. In the obtained solution, which is visualized in Fig.~\ref{fig:drone_muli_agent_map_n}, the first drone is assigned to visit three nearby sites while the other is dispatched to the last, distant site (UAV1:[S,1,2,3,S], UAV2:[S,4,S]). The charging station remained unused as both drones battery charge satisfied the minimum permissible SoC limit, which was set to $40\%$.


\begin{figure}[htp]
\centering
\subfloat[]{%
  {\includegraphics[clip,width=0.9\columnwidth]{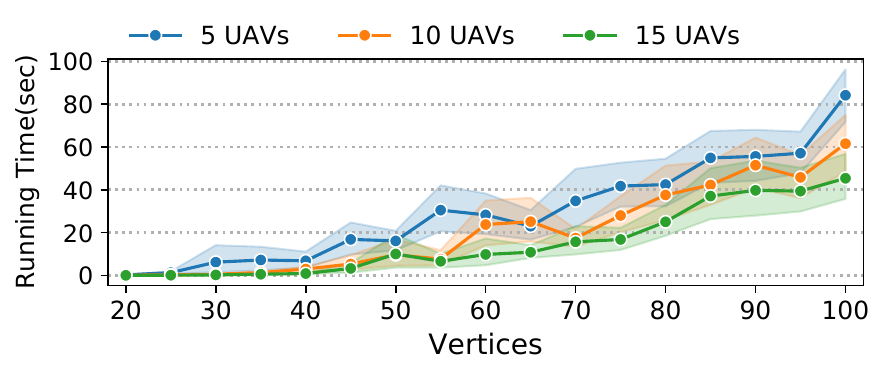}}%
}

\subfloat[]{%
  {\includegraphics[clip,width=0.9\columnwidth]{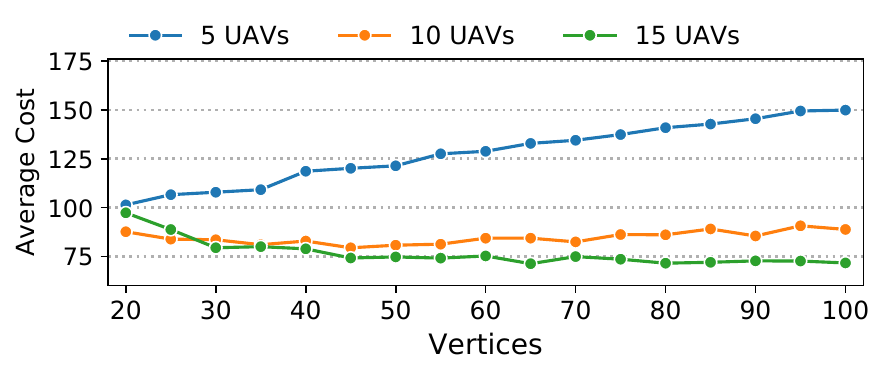}}%
}
\caption{Performance of the multi-UAV MILP program against the number of vertices: (a) Computational time across $30$ runs at 99\% confidence interval (b) Average tour cost per UAV.}
\label{fig:multi-agent-milp}
\end{figure}

\begin{figure}[htp]
    \centering
    {\includegraphics[width=\columnwidth]{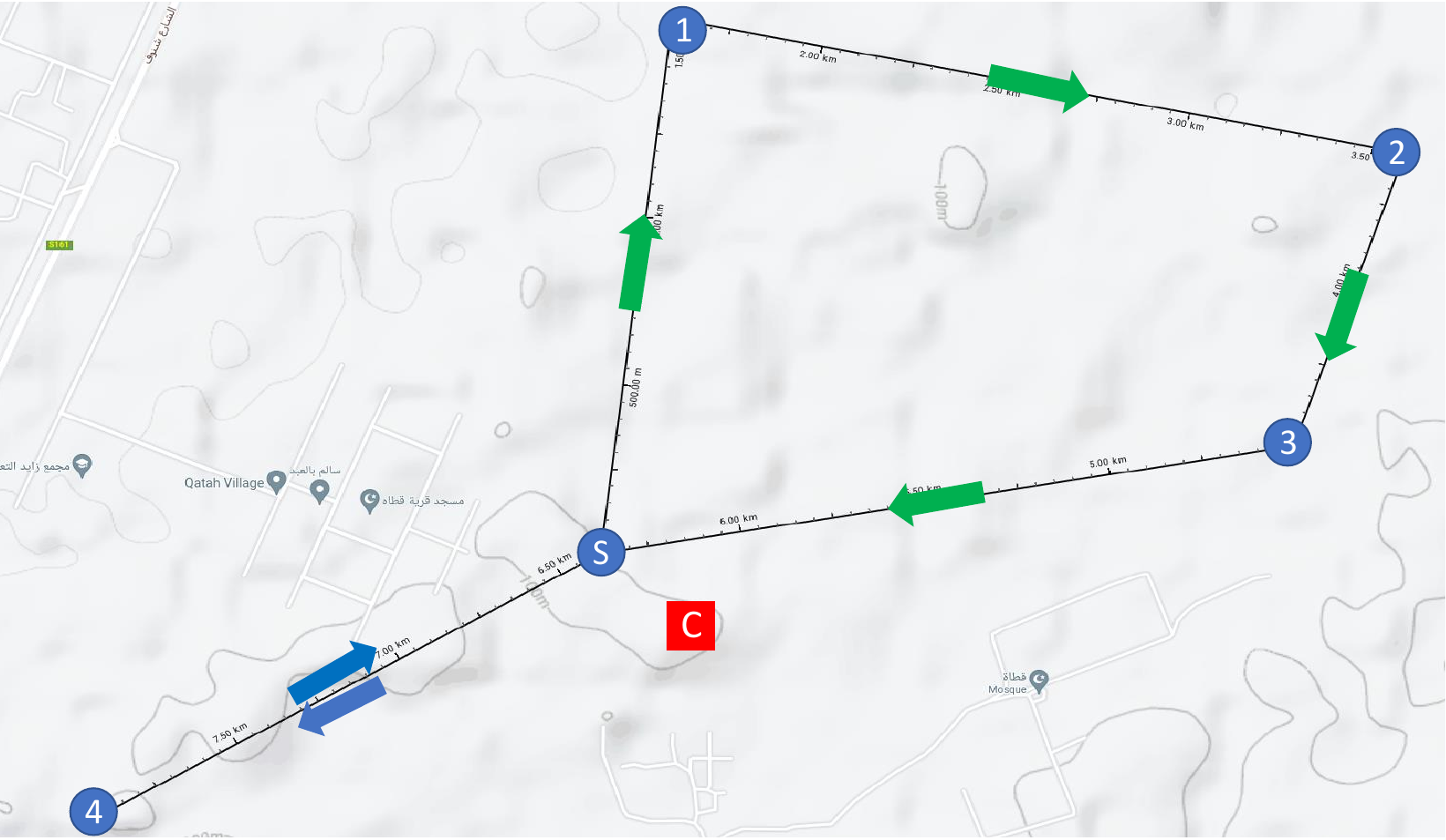}}
    \caption{Simulation of the real-world experiment in Sec.~\ref{exp} considering multi-UAV scenario with $2$ drones (the colored arrows trace the UAVs trajectories). }
    \label{fig:drone_muli_agent_map_n}
\end{figure}

\color{blue}

\color{blue}

\color{black}

\subsection{Proofs}

\begin{customlem}{\ref{lem:tot}}
In an optimal flight mission plan  $(\mathcal{F},b(\cdot))$, we have
$$
\underline{c}\cdot d(\mathcal{F})+c' \le 
\tau(\mathcal{F}) + \tau_{\rm c}(b(\mathcal{F}))
\le \overline{c}\cdot d(\mathcal{F})+c'
$$
where either 
\begin{enumerate}

\item[1)]
$\underline{c}=\overline{c}=c_{a}$ and ${c}'=0$, or 

\item[2)] $\underline{c}=c_{a}+\underline{c}_{f}c_{b}\frac{\eta_{\rm d}}{\eta_{\rm c}}$, $\overline{c}=c_{a}+\overline{c}_{f}c_{b}\frac{\eta_{\rm d}}{\eta_{\rm c}}$, and $c'=\frac{c_{b}}{\eta_{\rm c}}(\underline{B}-x_0)$. 

\end{enumerate}
\end{customlem}
\begin{proof}
Consider an optimal flight plan $(\mathcal{F},b(\cdot))$ and assume that the charging stations, in the order they appear on $\mathcal{F}$, is $\mathcal{F}_{i_1},\ldots, \mathcal{F}_{i_r}$, where without loss of generality, we assume $\mathcal{F}_{i_1}\ne v_0$. 
For completeness, let $i_0\triangleq1$ and $i_{r+1}\triangleq|\mathcal{F}|$. For $j=0,1,\ldots,r$, let $$D_j\triangleq \eta_{\rm d}\sum_{k=i_{j}}^{i_{j+1}-1} c_{f}(\mathcal{F}_k,\mathcal{F}_{k+1}) \cdot d(\mathcal{F}_k,\mathcal{F}_{k+1}),$$ and for $j=1,\ldots,r$, let
$B_j\triangleq \eta_{\rm c}b(\mathcal{F}_{i_j})$.

Then, the feasibility of the flight mission plan $\mathcal{F}$ implies
\begin{equation}\label{e1}
I(r) \triangleq x_0-\sum_{k=0}^jD_k+\sum_{k=1}^jB_k\geq\underline{B}, ~\text{for }j=0,\ldots,r
\end{equation}

Let us refer to Ineq.~\eqref{e1} for a particular $j$ as $I(j)\ge \underline{B}$.  
Particularly, consider $I(r)\ge \underline{B}$. Suppose that this inequality is not tight, that is, the left-hand side is strictly larger than the right-hand side. Note that the variable $b(\mathcal{F}_{i_r})=\frac{B_r}{ \eta_{\rm c}}$ appears only in this inequality. Since $b(\mathcal{F}_{i_r})$ appears in the objective function $\tau_{\rm c}(b(\mathcal{F}))$ with a positive coefficient (i.e., $\tau_{\rm c}(b(u))=c_{b}b(u)$), there are two cases: (i) $b(\mathcal{F}_{i_r})=0$ at optimality, if $I(r)>\underline{B}$, or  (ii) $b(\mathcal{F}_{i_r})> 0$ at optimality,  if $I(r)=\underline{B}$. Otherwise, it will contradict to the optimality of $b(\mathcal{F}_{i_r})$, by reducing the value of $b(\mathcal{F}_{i_r})$. If it is case (i), then the inequality $I(r-1)\ge \underline{B}$ becomes redundant (as $I(r-1)\ge I(r)>0$). Removing $I(r-1)\ge \underline{B}$, the variable $b(\mathcal{F}_{i_{r-1}})$ appears only in $I(r)\ge \underline{B}$. Similarly, we conclude that  $b(\mathcal{F}_{i_{r-1}})=0$ and remove the (now) redundant inequality $I(r-2)\ge \underline{B}$. 

Continuing this argument, we conclude that there are two cases: (1) either all variables $b(\mathcal{F}_{i_{j}})$ are set to zero in which case the value of the objective is $\tau(\mathcal{F})=c_{a} d(\mathcal{F})$, or (2) we have
$$
x_0-\sum_{k=0}^rD_k+\sum_{k=1}^rB_k=\underline{B},
$$     
In case (2), the value of the objective is 

\begin{align*}
\tau(\mathcal{F})+\frac{c_{b}}{\eta_{\rm c}}\sum_{k=1}^{r} B_{k} &  =\tau(\mathcal{F})+\frac{c_{b}}{\eta_{\rm c}}(\underline{B}-x_0+\sum_{k=0}^rD_k)\\
 & = \tau(\mathcal{F})+\frac{c_{b}}{\eta_{\rm c}}\sum_{k=0}^rD_k+\frac{c_{b}}{\eta_{\rm c}}(\underline{B}-x_0)
\end{align*}
Therefore,
\[
\scalemath{0.78}{
\begin{aligned}
\left(c_{a}+\underline{c}_{f}c_{b}\frac{\eta_{\rm d}}{\eta_{\rm c}}\right)d(\mathcal{F})+\frac{c_{b}}{\eta_{\rm c}}(\underline{B}-x_0)  & \le \tau(\mathcal{F})+\frac{c_{b}}{\eta_{\rm c}}\sum_{k=0}^rD_k+\frac{c_{b}}{\eta_{\rm c}}(\underline{B}-x_0) \\
& \le \left(c_{a}+\overline{c}_{f}c_{b}\frac{\eta_{\rm d}}{\eta_{\rm c}}\right)d(\mathcal{F})+\frac{c_{b}}{\eta_{\rm c}}(\underline{B}-x_0).
\end{aligned} }
\]
\normalsize{}
\end{proof}
\begin{customlem}{\ref{lem:red}}
Given any feasible flight mission plan $(\mathcal{F},b(\cdot))$, there is another feasible flight mission plan $(\mathcal{F},b'(\cdot))$ such that 
$$
\tau_{\rm c}(b'(\mathcal{F})) \le \frac{\underline{B}-x_0}{\eta_{\rm c}}+\frac{\overline{c}_{f}\eta_{\rm d}}{\eta_{\rm c}} \cdot d(\mathcal{F})$$ 
Such a plan $(\mathcal{F},b'(\cdot))$ can be found in $O(|V|)$ time.
\end{customlem}
\begin{proof}
The proof follows from Lemma~\ref{lem:tot}, which uses algorithm~${\tt Fix\mbox{-}charge}$ to implement the argument used in Lemma~\ref{lem:tot} to find a feasible fight mission plan.  
\end{proof}

\begin{customthm}{\ref{thm:main}}
	The flight plan $(\mathcal{F},b'(\cdot))$ returned by algorithm ${\tt Find\mbox{-}plan}\big[V,d]$ has cost $$\tau(\mathcal{F})+\tau_{\rm c}(b'(\mathcal{F}))=O({\sf OPT}_{\sf DFP})+O(1).$$
\end{customthm}
\begin{proof}
Let $(\mathcal{F}^*,b^*(\cdot))$ be an optimal flight plan for {\sf DFP}. Clearly, this plan can be trivially turned into a feasible solution $(\mathcal{F}^*,x)$ for {\sf SDFP} by setting $x_k=\overline{B}$ for all $\mathcal{F}_k \in \mathcal{C}$. It follows that 
\begin{equation}
\label{e2}
{\sf OPT}_{\sf SDFP}\le \widehat{d}(\mathcal{F}^*).
\end{equation}
On the other hand, Lemma~\ref{lem:GS} implies that
\begin{equation}
\label{e2-}
\widehat{d}(\mathcal{F})\le \left(\frac{(1+\alpha+\alpha\beta) \log(|T|))}{1-\alpha}\right) {\sf OPT}_{\sf SDFP}.
\end{equation}
Lemma~\ref{lem:tot} also implies that
\begin{equation}
\label{e3}
{\sf OPT}_{\sf DFP}\ge \underline{c}\cdot d(\mathcal{F}^*)+{c}',
\end{equation}
while Lemma~\ref{lem:red} implies that
\begin{equation}
\label{e4}
b'(\mathcal{F})\le\frac{\underline{B}-x_0}{\eta_{\rm c}}+\frac{\overline{c}_{f}\eta_{\rm d}}{\eta_{\rm c}}d(\mathcal{F}),
\end{equation}
and the definition of $\widehat{d}(\cdot,\cdot)$ implies
\begin{equation}
\label{e5} 
\underline{c}_{f} d(\mathcal{F})\le \widehat{d}(\mathcal{F})\text{ and }\widehat{d}(\mathcal{F}^*)\le\overline{c}_{f} d(\mathcal{F}^*).
\end{equation}
Putting together~\eqref{e2},~\eqref{e2-},~\eqref{e3},~\eqref{e4}, and~\eqref{e5}, we obtain
\[
\scalemath{0.82}{
\begin{aligned}
\tau(\mathcal{F})+\tau_{\rm c}(b'(\mathcal{F}))& = c_{a} d(\mathcal{F})+c_{b}b'(\mathcal{F})\\
&\le \left(c_{a}+\overline{c}_{f}c_{b}\frac{\eta_{\rm d}}{\eta_{\rm c}}\right)d(\mathcal{F})+c_{b}
\frac{\underline{B}-x_0}{\eta_{\rm c}} \\
&\le  \Psi\frac{\overline{c}_{f}}{\underline{c}_{f}}\left(c_{a}+\overline{c}_{f}c_{b}\frac{\eta_{\rm d}}{\eta_{\rm c}}\right) 
\left(\frac{{\sf OPT}_{\sf DFP}-{c}'}{\underline{c}}\right) +c_{b} \frac{\underline{B}-x_0}{\eta_{\rm c}},
\end{aligned}
}
\]
\noindent where $\Psi \triangleq \left(\frac{(1+\alpha+\alpha\beta) \log(|T|))}{1-\alpha}\right)$.

\end{proof}